\documentclass{article}

\def\TARGET{1}

\PassOptionsToPackage{table, svgnames, dvipsnames}{xcolor}

\usepackage[final]{neurips_2023}

\usepackage{graphicx}

\usepackage{amssymb,amsmath,amsfonts,amsthm,commath}
\usepackage{algorithm}
\usepackage{algorithmic}
\usepackage{hyperref}
\usepackage{natbib}
\usepackage{multirow}
\usepackage[capitalise]{cleveref}
\usepackage{enumitem}
\usepackage{wrapfig}
\usepackage{colortbl}

\newcommand{\gray}{\cellcolor[gray]{0.8}}

\usepackage[textsize=tiny]{todonotes}

\theoremstyle{plain}
\newtheorem{theorem}{Theorem}

\newtheorem{lemma}{Lemma}

\theoremstyle{definition}
\newtheorem{definition}{Definition}
\newtheorem{assumption}{Assumption}
\theoremstyle{remark}
\newtheorem{remark}{Remark}

\renewcommand{\t}{\text}
\newcommand{\op}[1]{\operatorname{#1}}
\newcommand{\C}[1]{{\mathcal{#1}}} 
\newcommand{\B}[1]{{\mathbb{#1}}} 
\newcommand{\BF}[1]{{\mathbf{#1}}} 
\newcommand{\F}[1]{{\mathfrak{#1}}}

\newcommand{\bra}{\langle}
\newcommand{\ket}{\rangle}

\newcommand{\tF}{\tilde{F}}
\newcommand{\x}{\BF{x}}
\newcommand{\y}{\BF{y}}
\newcommand{\z}{\BF{z}}
\newcommand{\vv}{\BF{v}}
\newcommand{\uu}{\BF{u}}
\newcommand{\g}{\BF{g}}

\newcommand{\tFgrad}{\nabla (\tF|_{\C{L}})}

\DeclareMathOperator*{\argmin}{argmin}
\DeclareMathOperator*{\argmax}{argmax}

\makeatletter
\def\blfootnote{\xdef\@thefnmark{}\@footnotetext}
\makeatother

\title{A Unified Approach for Maximizing Continuous DR-submodular Functions}
\date{}
\author{
\begin{tabular}{c}
Mohammad Pedramfar \\
{\normalfont Purdue University} \\
\texttt{mpedramf@purdue.edu}
\end{tabular}
\And 
\begin{tabular}{c}
Christopher John Quinn \\
{\normalfont Iowa State University} \\
\texttt{cjquinn@iastate.edu}
\end{tabular}
\And
\begin{tabular}{c}
Vaneet Aggarwal \\
{\normalfont Purdue University} \\
\texttt{vaneet@purdue.edu}
\end{tabular}
}

\begin{document}

\maketitle

\begin{abstract}
This paper presents a unified approach for maximizing continuous DR-submodular functions that encompasses a range of settings and oracle access types. Our approach includes a Frank-Wolfe type offline algorithm for both monotone and non-monotone functions, with different restrictions on the general convex set. We consider settings where the oracle provides access to either the gradient of the function or only the function value, and where the oracle access is either deterministic or stochastic. We determine the number of required oracle accesses in all cases. Our approach gives new/improved results for nine out of the sixteen considered cases, avoids computationally expensive projections in three cases, with the proposed framework matching performance of state-of-the-art approaches in the remaining four cases. Notably, our approach for the stochastic function value-based oracle enables the first regret bounds with bandit feedback for stochastic DR-submodular functions.
\end{abstract}

\newcommand{\DrawOfflineTabular}{
\begin{tabular}{ | c | c | c | c | c | c | c | c | c | }
\hline
$F$ & Set & Oracle & Setting & Reference & Appx.  & Complexity \\
\hline
\multirow{11}{*}{\rotatebox{90}{Monotone}}
& \multirow{5}*{\rotatebox{90}{$0 \in \C{K}$}}
  & \multirow{3}*{$\nabla F$}
    & \multirow{1}*{det.}
      & \cite{bian17_guaran_non_optim}, {\color{blue} (*)}
        & $1-1/e$ & $O(1/\epsilon)$ \\
    \cline{4-7}
& & & \multirow{2}*{stoch.}
      & \cite{mokhtari20_stoch_condit_gradien_method}, {\color{blue} (*)}
        & $1-1/e$ & $O(1/\epsilon^3)$ \\
& & & & \cite{zhang22_stoch_contin_submod_maxim} $\ddagger$
        & $1-1/e$ & $O(1/\epsilon^2)$ \\
  \cline{3-7}
& & \multirow{2}*{$F$}
    & \multirow{1}*{det.}
      & \gray {\color{blue}This paper}
        &\gray  $1-1/e$ & \gray $O(1/\epsilon^3)$ \\
    \cline{4-7}
& & & \multirow{1}*{stoch.}
      & \gray {\color{blue}This paper} & \gray $1-1/e$ & \gray $O(1/\epsilon^5)$ \\
\cline{2-7}
& \multirow{6}*{\rotatebox{90}{general$\dagger$}}
  & \multirow{4}*{$\nabla F$}
    & \multirow{2}*{det.}
      & \cite{hassani17_gradien_method_submod_maxim} $\ddagger$
        & $1/2$    & $O(1/\epsilon)$ \\
& & & &  {\color{blue}This paper}
        &  $1/2$    & $\tilde{O}(1/\epsilon)$ \\
    \cline{4-7}
& & & \multirow{2}*{stoch.}
      & \cite{hassani17_gradien_method_submod_maxim}$\ddagger$
        & $1/2$    & $O(1/\epsilon^2)$ \\
& & & & {\color{blue}This paper}
        & $1/2$    & $\tilde{O}(1/\epsilon^3)$ \\
  \cline{3-7}
& & \multirow{2}*{$F$} 
    & \multirow{1}*{det.}  
      & \gray {\color{blue}This paper}
        & \gray $1/2$    & \gray $\tilde{O}(1/\epsilon^3)$ \\
    \cline{4-7}
& & & \multirow{1}*{stoch.}  
      & \gray {\color{blue}This paper}
        & \gray $1/2$    &  \gray $\tilde{O}(1/\epsilon^5)$ \\
\hline
\multirow{11}*{\rotatebox{90}{Non-Monotone}}
& \multirow{5}*{\rotatebox{90}{d.c.}}
  & \multirow{2}*{$\nabla F$}
    & \multirow{1}*{det.}
      & \cite{bian17_contin_dr_maxim}, {\color{blue} (*)}
        & $1/e$   & $O(1/\epsilon)$ \\
    \cline{4-7}
& & & \multirow{1}*{stoch.}
      & \cite{mokhtari20_stoch_condit_gradien_method}, {\color{blue} (*)}
        & $1/e$   & $O(1/\epsilon^3)$ \\
  \cline{3-7}
& & \multirow{2}*{$F$}
    & \multirow{1}*{det.}
      &\gray {\color{blue}This paper}
        &\gray $1/e$   &\gray $O(1/\epsilon^3)$ \\
    \cline{4-7}
& & & \multirow{1}*{stoch.}
      &\gray {\color{blue}This paper}
        &\gray $1/e$   &\gray $O(1/\epsilon^5)$ \\
\cline{2-7}
& \multirow{6}*{\rotatebox{90}{general}}
  & \multirow{4}*{$\nabla F$}
    & \multirow{3}*{det.}
      & \cite{durr19_nonmonotone}
        & $\frac{1-h}{3\sqrt{3}}$ & $O(e^{\sqrt{dL/\epsilon}})$ \\
        & & & & \cite{du2022improved}
                & $\frac{1-h}{4}$         & $O(e^{\sqrt{dL/\epsilon}})$ \\
        & & & & \cite{du22_lyapun}, {\color{blue} (*)}
                & $\frac{1-h}{4}$         & $O(1/\epsilon)$ \\
    \cline{4-7}
    & & & \multirow{1}*{stoch.}  
          & \gray {\color{blue}This paper}
            & \gray $\frac{1-h}{4}$          & \gray $O(1/\epsilon^3)$ \\
\cline{3-7}
& & \multirow{2}*{$F$}
    & \multirow{1}*{det.}  
      & \gray {\color{blue}This paper}
        & \gray $\frac{1-h}{4}$          & \gray $O(1/\epsilon^3)$ \\
    \cline{4-7}
& & & \multirow{1}*{stoch.} 
      & \gray {\color{blue}This paper}
        & \gray $\frac{1-h}{4}$          & \gray $O(1/\epsilon^5)$ \\
\hline
\end{tabular}
}

\newcommand{\DrawOfflineTable}{
\begin{table}[H]
\small
\caption{Offline DR-submodular optimization results.  } 
\label{TBL:offline}
\begin{center}
\DrawOfflineTabular
\end{center}
{~\\
This table compares the different results for the number of oracle calls (complexity) \textit{within the feasible set}  for DR-submodular maximization. 
Shaded rows indicate problem settings for which our work has the \textbf{first guarantees} or \textbf{beats the SOTA}.
The rows marked with a blue star {\color{blue}(*)} correspond to cases where \cref{ALG:main_offline} \textbf{generalizes the corresponding algorithm} and therefore has the same performance. 
The different columns enumerate properties of the function, the convex feasible region (downward-closed, includes the origin, or general), and the oracle, as well as the approximation ratios and oracle complexity (the number of queries needed to achieve the stated approximation ratio with at most $\epsilon>0$ additive error).
We have $h := \min_{\x \in \C{K}} \|x\|_\infty$.
(See \cref{APX:constraint_vs_query} regarding~\cite{mokhtari20_stoch_condit_gradien_method} and~\cite{zhang22_stoch_contin_submod_maxim}).
$\dagger$ when the oracle can be queried for any points in $[0,1]^d$ (even outside the feasible region $\C{K}$), the problem of optimizing monotone DR-submodular functions over a general convex set simplifies ---   \cite{bian17_guaran_non_optim} and \cite{mokhtari20_stoch_condit_gradien_method} achieve the same ratios and complexity bounds as listed above for $0\in\C{K}$; 
\cite{chen20_black_box_submod_maxim} can achieve an approximation ratio of $1-1/e$ with the $O(1/\epsilon^3)$ and $O(1/\epsilon^5)$  complexity for exact and stochastic value oracles respectively.  
$\ddagger$ \cite{hassani17_gradien_method_submod_maxim} and~\cite{zhang22_stoch_contin_submod_maxim} use gradient ascent, requiring potentially computationally expensive projections. 
}
\end{table}
}

\newcommand{\DrawOnlineTabular}{
\begin{tabular}{ | c | c | c |  c | c | c | c | }
\hline
$F$ & Set & Feedback &  Reference & Coef. $\alpha$ & $\alpha$-Regret \\
\hline
\multirow{6}*{ \rotatebox{90}{Monotone} }
& \multirow{3}*{$0 \in \C{K}$}
  & \multirow{3}*{$\nabla F$}
    & \cite{chen18_projec_free_onlin_optim_stoch_gradien}$\dagger$,
      & $1/e$   & $O(T^{2/3})$ \\
& & & \gray {\color{blue}This paper}
      & \gray $1-1/e$ & \gray $O(T^{3/4})$ \\
  \cline{4-6}
\cline{3-6}
& & \multirow{1}*{$F$}
    &\gray {\color{blue}This paper}
      &\gray $1-1/e$ &\gray $O(T^{5/6})$ \\
  \cline{4-6}
\cline{2-6}
& \multirow{3}*{{general}}
  & \multirow{2}*{$\nabla F$}
    & \cite{hassani17_gradien_method_submod_maxim} $\ddagger$ 
      & $1/2$    & $O(T^{1/2})$ \\
    & & & {\color{blue}This paper}
      & $1/2$     & $\tilde{O}(T^{3/4})$ \\
  \cline{4-6}
\cline{3-6}
& & \multirow{1}*{$F$} 
    & {\color{blue}This paper} \gray
      &\gray $1/2$   &\gray $\tilde{O}(T^{5/6})$ \\
  \cline{4-6}
\hline
\multirow{4}*{ \rotatebox{90}{Non-mono.} }
& \multirow{2}*{{d.c.}}
  &\multirow{1}*{$\nabla F$}
    &\gray {\color{blue}This paper}
      &\gray $1/e$   &\gray $O(T^{3/4})$ \\
  \cline{4-6}
\cline{3-6}
& & \multirow{1}*{$F$} 
    &\gray {\color{blue}This paper}
      &\gray $1/e$   &\gray $O(T^{5/6})$ \\
  \cline{4-6}
\cline{2-6}
& \multirow{2}*{{general} }
  & \multirow{1}*{$\nabla F$}
    &\gray {\color{blue}This paper}
      &\gray $\frac{1-h}{4}$   &\gray $O(T^{3/4})$ \\
  \cline{4-6}
\cline{3-6}
& & \multirow{1}*{$F$}
    & {\color{blue}This paper} \gray
      &\gray $\frac{1-h}{4}$   &\gray $O(T^{5/6})$ \\
\cline{4-6}
\hline
\end{tabular}
}

\newcommand{\DrawOnlineTableShort}{
\begin{table}[H]
\caption{Online stochastic DR-submodular optimization.}
\label{TBL:online:short}
\begin{center}
\DrawOnlineTabular
\end{center}
{~\\
\small 
This table compares the different results for the expected $\alpha$-regret for online stochastic DR-submodular maximization for the under bandit and semi-bandit feedback. Shaded rows indicate problem settings for which our work has the \textbf{first guarantees} or \textbf{beats the SOTA}.
We have $h := \min_{\x \in \C{K}} \|x\|_\infty$.
$\dagger$ the analysis in \cite{chen18_projec_free_onlin_optim_stoch_gradien} has an error (see the supplementary material for details). 
$\ddagger$ \cite{hassani17_gradien_method_submod_maxim} uses gradient ascent, requiring potentially computationally expensive projections. 
}
\end{table}
} 

\blfootnote{This work was supported in part by the National Science Foundation under grants CCF-2149588 and CCF-2149617, and Cisco, Inc.}

\section{Introduction}

The problem of optimizing DR-submodular functions over a convex set 
has attracted considerable interest in both the machine learning and theoretical computer science communities \citep{bach2019submodular,bian2019optimal,hassani17_gradien_method_submod_maxim,niazadeh20_optim_algor_contin_non_submod}. 
This is due to its many practical applications in modeling real-world problems, such as influence/revenue maximization, facility location, and non-convex/non-concave quadratic programming \citep{bian17_contin_dr_maxim,djolonga2014map,ito2016large,gu2023profit,li2023experimental}.  
as well as more recently identified applications like serving heterogeneous learners under networking constraints~\cite{li2023experimental} and joint optimization of routing and caching in networks~\cite{li23_joint_optim_routi}.

Numerous studies investigated developing approximation algorithms for constrained DR-submodular maximization, utilizing a variety of algorithms and proof analysis techniques.
These studies have addressed both monotone and non-monotone functions and considered various types of constraints on the feasible region.
The studies have also considered different types of oracles---gradient oracles and value oracles, where the oracles could be exact (deterministic) or stochastic.
Lastly, for some of the aforementioned offline problem settings, some studies have also considered analogous online optimization problem settings as well, where performance is measured in regret over a horizon. 
This paper aims to unify the disparate offline problems under a single framework by providing a comprehensive algorithm and analysis approach that covers a broad range of setups. By providing a unified framework, this paper presents novel results for several cases where previous research was either limited or non-existent, both for offline optimization problems and extensions to related stochastic online optimization problems.

This paper presents a Frank-Wolfe based meta-algorithm for (offline) constrained DR-submodular maximization where we could only query within the constraint set, with sixteen variants for sixteen problem settings.
The algorithm is designed to handle settings where (i) the function is monotone or non-monotone, (ii) the feasible region is a downward-closed (d.c.) set (extended to include $0$ for monotone functions) or a general convex set, (iii) gradient or value oracle access is available, and (iv) the oracle is exact or stochastic. 
Table 1 enumerates the cases and corresponding results on oracle complexity (further details are provided in Appendix \ref{sec:related}).
We derive the first oracle complexity guarantees for nine cases,  
derive the oracle complexity in three cases where previous result had a computationally expensive projection step \cite{zhang22_stoch_contin_submod_maxim,hassani17_gradien_method_submod_maxim} (and we obtain matching complexity in one of these), and obtain matching guarantees in the remaining four cases.
\begin{figure}[t]
\DrawOfflineTable
\end{figure}
In addition to proving approximation ratios and oracle complexities for  several (challenging) settings that are the first or improvements over the state of the art, the \textit{technical novelties of our approach} include: 
\begin{enumerate}[label=(\roman*)]
\item A new construction procedure of a shrunk constraint set that allows us to work with lower dimensional feasible sets when given a value oracle, resulting in the first results on general lower dimensional feasible sets given a value oracle.
\item The first Frank-Wolfe type algorithm for analyzing monotone functions over a general convex set for any type of oracle, where only feasible points can be queried. 
\item Shedding light on a previously unexplained gap in approximation guarantees for monotone DR-submodular maximization. 
Specifically, by considering the notion of query sets and assuming that the oracles can only be queries within the constraint set, we divide the class of monotone submodular maximization into monotone submodular maximization over convex sets containing the origin and monotone submodular maximization over general convex sets.
Moreover, we conjecture that the $1/2$ approximation coefficient, which has been considered sub-optimal in the literature, is optimal when oracle queries can only be made within the constraint set. 
(See Appendix~\ref{APX:constraint_vs_query} for more details.)
\end{enumerate}

Furthermore, we also consider online stochastic DR-submodular optimization with bandit feedback, where an agent sequentially picks actions (from a convex feasible region), receives stochastic rewards (in expectation a DR-submodular function) but no additional information, and seeks to maximize the expected cumulative reward.
Performance is measured against the best action in expectation 
(or a near-optimal baseline when the offline problem is NP-hard but can be approximated to within $\alpha$ in polynomial time), the difference denoted as expected $\alpha$-regret. 
For such problems, when only bandit feedback is available (it is typically a strong assumption that semi-bandit or full-information feedback is available), the agent must be able to learn from stochastic value oracle queries over the feasible actions action. 
By designing offline algorithms that only query feasible points, we made it possible to convert those offline algorithms into online algorithms. 
In fact, because of how we designed the offline algorithms, we are able to access them in a black-box fashion for online problems when only bandit feedback is available. 
Note that previous works on DR-submodular maximization with bandit feedback in monotone settings (e.g.~\cite{zhang19_onlin_contin_submod_maxim},~\cite{niazadeh21_onlin_learn_offlin_greed_algor} and~\cite{wan23_bandit_multi_dr_submod_maxim}) explicitly assume that the convex set contains the origin.

For each of the offline setups, we extend the offline algorithm (the respective variants for stochastic value oracle) and oracle query guarantees to provide algorithms and $\alpha$-regret bounds in the bandit feedback scenario. Table 2 enumerates the problem settings and expected regret bounds with bandit and semi-bandit feedback. 
The key contributions of this work can be summarized as follows:

\noindent {\bf 1.} 
This paper proposes a unified approach for maximizing continuous DR-submodular functions in a range of settings with different oracle access types, feasible region properties, and function properties.
A Frank-Wolfe based algorithm is introduced, which compared to SOTA methods for each of the sixteen settings, achieves the best-known approximation coefficients for each case while providing 
(i) the first guarantees in nine cases,
(ii) reduced computational complexity by avoiding projections in 
three cases, and 
(iii) matching guarantees in remaining four cases.

\noindent {\bf 2.} 
In particular, this paper gives the first results on offline DR-submodular maximization (for both monotone and non-monotone functions) over general convex sets and even for downward-closed convex sets, when only a value oracle is available over the feasible set.
Most prior works on offline DR-submodular maximization require access to a gradient oracle.

\if\TARGET0
    \begin{wrapfigure}{r}{0.69\textwidth}
    \vspace{-.3in}
    \begin{minipage}{0.69\textwidth}
    { \small 
    \DrawOnlineTableShort
    }
    \vspace{-0.3in}
    \end{minipage}
    \end{wrapfigure}
\else
    \begin{figure}
    \vspace{-.3in}
    { \small 
    \DrawOnlineTableShort
    }
    \end{figure}
\fi
\noindent {\bf 3.}
The results, summarized in Table 2, are presented with two feedback models---bandit feedback where only the (stochastic) reward value is available and semi-bandit feedback where a single stochastic sample of the gradient at the location is provided. This paper presents the first regret analysis with bandit feedback for stochastic DR-submodular maximization for both monotone and non-monotone functions. For semi-bandit feedback case, we provide the first result in one case, improve the state of the art result in two cases, and gives the result without computationally intensive projections in one case.  

\paragraph{Related Work:} The key related works are summarized in Tables 1 and 2. 
We briefly discuss some key works here; see the supplementary materials for more discussion. 
For online DR-submodular optimization with bandit feedback, there has been some prior works in the adversarial setup \cite{zhang19_onlin_contin_submod_maxim,zhang23_onlin_learn_non_submod_maxim,niazadeh21_onlin_learn_offlin_greed_algor,wan23_bandit_multi_dr_submod_maxim} which are not included in Table~2 as we consider the stochastic setup.
\cite{zhang19_onlin_contin_submod_maxim} considered monotone DR-submodular functions over downward-closed convex sets and achieved $(1-1/e)$-regret of $O(T^{8/9})$ in adversarial setting.
This was later improved by~\cite{niazadeh21_onlin_learn_offlin_greed_algor} to $O(T^{5/6})$.
\cite{wan23_bandit_multi_dr_submod_maxim} further improved the regret bound to $O(T^{3/4})$.
However, it should be noted that they use a convex optimization subroutine at each iteration which could be even more computationally expensive than projection.
\cite{zhang23_onlin_learn_non_submod_maxim} considered non-monotone DR-submodular functions over downward-closed convex sets and achieved $1/e$-regret of $O(T^{8/9})$ in adversarial setting.

\section{Background and Notation}

We introduce some basic notions, concepts and assumptions which will be used throughout the paper. 
For any vector $\x \in \B{R}^d$, $[\x]_i$  is the $i$-th entry of $\x$.
We consider the partial order on $\B{R}^d$ where $\x \leq \y$ if and only if $[\x]_i \leq [\y]_i$ for all $1 \leq i \leq d$.
For two vectors $\x, \y \in \B{R}^d$, the \emph{join} of $\x$ and $\y$, denoted by $\x \vee \y$ and the \emph{meet} of $\x$ and $\y$, denoted by $\x \wedge \y$, are defined by 
\begin{equation}
\x \vee \y := ( \max\{ [\x]_i, [\y]_i \} )_{i = 1}^d
\quad\t{ and }\quad
\x \wedge \y := ( \min\{ [\x]_i, [\y]_i \} )_{i = 1}^d,
\end{equation}
respectively.
Clearly, we have 
$\x \wedge \y \leq \x \leq \x \vee \y$.
We use $\| \cdot \|$ to denote the Euclidean norm, 
and $\| \cdot \|_\infty$ to denote the supremum norm.
In the paper, we consider a bounded convex domain $\C{K}$ and w.l.o.g. assume that
$\C{K} \subseteq [0, 1]^{d}$. 
We say that $\C{K}$ is \emph{down-closed} (d.c.) if there is a point $\BF{u} \in \C{K}$ such that for all 
$\z \in \C{K}$, we have $ \{ \x \mid \BF{u} \leq \x \leq z \} \subseteq \C{K}$.
The \emph{diameter} $D$ of the convex domain $\C{K}$ is defined as 
$D := \sup_{\x, \y \in \C{K}} \norm{\x - \y}$. 
We use $\B{B}_r(x)$ to denote the open ball of radius $r$ centered at $\x$.
More generally, for a subset $X \subseteq \B{R}^d$, we define $\B{B}_r(X) := \bigcup_{x \in X} \B{B}_r(x)$.
If $A$ is an affine subspace of $\B{R}^d$, then we define $\B{B}_r^A(X) := A \cap \B{B}_r(X)$.
For a function $F : \C{D} \to \B{R}$ and a set $\C{L}$, we use $F|_{\C{L}}$ to denote the restriction of $F$ to the set $\C{D} \cap \C{L}$.
For a linear space $\C{L}_0 \subseteq \B{R}^d$, we use $P_{\C{L}_0} : \B{R}^d \to \C{L}_0$ to denote the projection onto $\C{L}_0$.
We will use $\B{R}^d_+$ to denote the set $\{ \x \in \B{R}^d | \x \geq 0 \}$.
For any set $X \subseteq \B{R}^d$, the affine hull of $X$, denoted by $\op{aff}(X)$, is defined to be the intersection of all affine subsets of $\B{R}^d$ that contain $X$.
The \textit{relative interior} of a set $X$ is defined by
\[
\op{relint}(X) := \{ \x \in X \mid \exists \varepsilon > 0, \B{B}_{\varepsilon}^{\op{aff}(X)}(\x) \subseteq X \}.
\] 
It is well known that for any non-empty convex set $\C{K}$, the set $\op{relint}(\C{K})$ is always non-empty.
We will always assume that the feasible set contains at least two points and therefore $\op{dim}(\op{aff}(\C{K})) \geq 1$, otherwise the optimization problem is trivial and there is nothing to solve.

A set function $f: \{0,1\}^{d} \rightarrow \B{R}^+$ is called \emph{submodular} if
for all $\x,\y \in \{0,1\}^{d}$ with $\x \geq \y$, we have
\begin{align}   \label{def:sub}
f(\x \vee \BF{a}) - f(\x) \leq f(\y \vee \BF{a}) - f(\y)
    ,\qquad \forall \BF{a} \in \{0, 1\}^{d}.
\end{align}
Submodular functions can be generalized over continuous domains. 
A function $F: [0,1]^{d} \rightarrow \mathbb{R}^+$ is called \emph{DR-submodular} if for all vectors
$\x,\y \in [0,1]^{d}$ with $\x \leq \y$, any basis vector $\BF{e}_{i} = (0,\cdots,0,1,0,\cdots,0)$ and any constant $c > 0$ such that $\x + c \BF{e}_{i} \in [0,1]^{d}$ and $\y + c \BF{e}_{i} \in [0,1]^{d}$, it holds that
\begin{align}	\label{def:DR-sub}
	F(\x + c \BF{e}_{i}) - F(\x) \geq F(\y + c \BF{e}_{i}) - F(\y).
\end{align}
Note that if function $F$ is differentiable then the diminishing-return (DR) property (\ref{def:DR-sub}) is equivalent to $\nabla F(\x) \geq \nabla F(\y)$ for $\x \leq \y \text{ with } \x,\y \in [0,1]^{d}.$
A function $F: \C{D} \to \mathbb{R}^{+}$ is \textit{$G$-Lipschitz continuous} if for all $\x, \y \in \C{D}$, $\| F(\x) - F(\y) \| \leq G \| \x - \y \|$.
A differentiable function $F: \C{D} \to \B{R}^+$ is \textit{$L$-smooth} if for all $\x, \y \in \C{D}$, $\| \nabla F(\x) - \nabla F(\y) \| \leq L \| \x -\y \|$.

A (possibly randomized) offline algorithm is said to be an $\alpha$-approximation algorithm (for constant $\alpha \in (0, 1]$) with $\epsilon \geq 0$ additive error for a class of maximization problems over non-negative functions if, for any problem instance $\max_{\z \in \C{K}} F(\z)$, the algorithm output $\x$ that satisfies the following relation with the optimal solution $\z^*$
\begin{align}\label{def:appxalg}
    \alpha F(\z^*) - \B{E}[ F(\x) ] \leq \epsilon,
\end{align} 
where the expectation is with respect to the (possible) randomness of the algorithm. Further, we assume an oracle that can query the value $F(\x) $ or the gradient $\nabla F(\x)$. The number of calls to the oracle to achieve the error in \eqref{def:appxalg} is called the \textit{evaluation complexity}.

\section{Offline Algorithms and Guarantees}

In this section, we consider the problem of maximizing a DR-submodular function over a general convex set in sixteen different cases, enumerated in \cref{TBL:offline}.
After setting up the problem in \cref{sec:offline:setup}, we then explain two key elements of our proposed algorithm when we only have access to a value oracle, (i) the \textsf{Black Box Gradient Estimate (BBGE)} procedure (Algorithm~\ref{ALG:gradient_estimate}) to balance bias and variance in estimating gradients (\cref{sec:offline:BBGE}) and (ii) the construction of a shrunken feasible region to avoid infeasible value oracle queries during the BBGE procedure (\cref{SS:construction}). Our main algorithm is proposed in \cref{sec:offline:alg} and analyzed in \cref{sec:offline:analysis}. 

\subsection{Problem Setup}\label{sec:offline:setup}

We consider  a general \emph{non-oblivious}
constrained stochastic optimization problem
\begin{align}
\max_{\z \in \C{K}} F(\z) := \max_{\z \in \C{K}} \B{E}_{\x \sim p(\x; \z)} [\hat{F}(\z, \x)],
\end{align}
where $F$ is a DR-submodular function, and $\hat{F}: \C{K} \times \F{X} \to \B{R}$ is determined by $\z$ and the random variable $\x$ which is independently sampled according to $\x \sim p(\x; \z)$.
We say the oracle has variance $\sigma^2$ if
$
\sup_{\z \in \C{K}} \op{var}_{\x \sim p(\x; \z)} [\hat{F}(\z, \x)] = \sigma^2.
$
In particular, when $\sigma = 0$, then we say we have access to an exact (deterministic) value oracle.
We will use $\hat{F}(\z)$ to denote the random variables $\hat{F}(\z, \x)$ where $\x$ is a random variable with distribution $p(\dot; \z)$.
Similarly, we say we have access to a stochastic gradient oracle if we can sample from function $\hat{G}: \C{K} \times \F{Y} \to \B{R}$ such that
$
\nabla F(\z) = \B{E}_{\y \sim q(\y; \z)} [\hat{G}(\z, \y)],
$
and $\hat{G}$ is determined by $\z$ and the random variable $\y$ which is sampled according to $\y \sim q(\y; \z)$.
Note that oracles are only defined on the feasible set.
We will use $\hat{G}(\z)$ to denote the random variables $\hat{G}(\z, \y)$ where $\y$ is a random variable with distribution $q(\dot; \z)$.

\begin{assumption} \label{assumptions}
We assume that $F : [0, 1]^d \to \B{R}$ is DR-submodular, first-order differentiable, non-negative, $G$-Lipschitz for some $G<\infty$, and $L$-smooth for some $L<\infty$.
We also assume the feasible region $\C{K}$ is a closed convex domain in $ [0, 1]^{d}$ with at least two points.
Moreover, we also assume that we either have access to a value oracle with variance $\sigma_0^2 \geq 0$ or a gradient oracle with variance $\sigma_1^2 \geq 0$.
\end{assumption}

\begin{remark}\label{remark:assumptions}
The proposed algorithm does not need to know the values of $L$, $G$, $\sigma_0$ or $\sigma_1$.
However, these constants appear in the final expressions of the number of oracle calls and the regret bounds.
\end{remark}

\subsection{Black Box Gradient Estimate}\label{sec:offline:BBGE}

Without access to a gradient oracle (i.e., first-order information), we estimate gradient information using samples from a value oracle. We will use a variation of the ``smoothing trick" technique \cite{flaxman2005online, hazan2016introduction,agarwal2010optimal,shamir17_optim_algor_bandit_zero_order,zhang19_onlin_contin_submod_maxim,chen20_black_box_submod_maxim,zhang23_onlin_learn_non_submod_maxim}, which involves averaging through spherical sampling around a given point.

\begin{definition} [Smoothing Trick] \label{def:smooth_definition}
For a function $F : \C{D} \to \B{R}$ defined on $\C{D} \subseteq \B{R}^d$, its $\delta$-smoothed version $\tilde{F}_\delta$ is 
given as 
\begin{equation}
    \label{eq:smooth_definition}
    \tilde{F}_\delta(\x) 
    := \B{E}_{\z \sim \B{B}_\delta^{\op{aff}(\C{D})}(\x)}[F(\z)] 
    = \B{E}_{\vv \sim \B{B}_1^{\op{aff}(\C{D}) - \x}(0)}[F(\x + \delta \vv)],
\end{equation}
where $\vv$ is chosen uniformly at random from the $\op{dim}(\op{aff}(\C{D}))$-dimensional  ball $\B{B}_1^{\op{aff}(\C{D}) - \x}(0)$. 
Thus, the function value $\tilde{F}_\delta(\x)$ is obtained by ``averaging'' $F$ over a sliced ball of radius $\delta$ around $\x$.
\end{definition}

When the value of $\delta$ is clear from the context, we may drop the subscript and simply use $\tilde{F}$ to denote the smoothed version of $F$.
It can be easily seen that if $F$ is DR-submodular, $G$-Lipschitz continuous, and $L$-smooth, then so is $\tilde{F}$ and
$\| \tilde{F}(\x) - F(\x) \| \le \delta G$, for any point in the domain of both functions.
Moreover, if $F$ is monotone, then so is $\tilde{F}$ (Lemma~\ref{L:smooth_approx}). 
Therefore $\tilde{F}_\delta$ is an approximation of the function $F$.
A maximizer of $\tilde{F}_\delta$ also maximizes $F$ approximately.

Our definition of smoothing trick differs from the standard usage by accounting for the affine hull containing $\C{D}$.
This will be of particular importance when the feasible region is of (affine) dimension less than $d$, such as when there are equality constraints. 
When $\op{aff}(\C{D}) = \B{R}^d$, our definition reduces to the standard definition of the smoothing trick. In this case, it is well-known that the gradient of the smoothed function $\tilde{F}_\delta$ admits an unbiased one-point estimator~\cite{flaxman2005online, hazan2016introduction}. 
Using a two-point estimator instead of the one-point estimator results in smaller variance \cite{agarwal2010optimal,shamir17_optim_algor_bandit_zero_order}.
In Algorithm 1,  we adapt the two-point estimator to the general setting.

\subsection{Construction of \texorpdfstring{$\C{K}_\delta$}{Shrunk Constraint Set}}\label{SS:construction}

\begin{wrapfigure}{r}{0.54\textwidth}
\vspace{-.35in}
\begin{minipage}{0.54\textwidth}
\begin{algorithm}[H]
\caption{\textsf{Black Box Gradient Estimate (BBGE)}}\label{ALG:gradient_estimate}
\begin{algorithmic}[1]
\STATE {\bfseries Input:} Point $\z$, sampling radius $\delta$, constraint linear space $\C{L}_0$, $k = \op{dim}(\C{L}_0)$, batch size $B$
\STATE Sample $\uu_1, \cdots, \uu_B$ i.i.d.\ from $S^{d-1} \cap \C{L}_0$
\STATE For $i=1$ to $B$, let $\y_i^+ \gets \z+\delta \uu_i, y_i^- \gets \z - \delta \uu_i$, and evaluate $\hat{F}(\y_i^+), \hat{F}(\y_i^-)$
\STATE $\g \gets \frac{1}{B}\sum_{i=1}^{B} \frac{k}{2\delta}\left[ \hat{F}(\y_i^+)-\hat{F}(\y_i^-) \right] \uu_i$
\STATE Output $\g$
\end{algorithmic}
\end{algorithm}
\end{minipage}
\end{wrapfigure}

We want to run Algorithm 1 as a subroutine within the main algorithm to estimate the gradient.
However, in order to run Algorithm 1, we need to be able to query the oracle within the set $\B{B}_\delta^{\op{aff}(\C{K})}(\x)$.
Since the oracle can only be queried at points within the feasible set, we need to restrict our attention to a set $\C{K}_\delta$ such that $\B{B}_\delta^{\op{aff}(\C{K})}(\C{K}_\delta) \subseteq \C{K}$. On the other hand, we want the optimal point of $F$ within $\C{K}_\delta$ to be close to the optimal point of $F$ within $\C{K}$.
One way to ensure that is to have $\C{K}_\delta$ not be too small.
More formally, we want that
$\B{B}_{\delta'}^{\op{aff}(\C{K})}(\C{K}_\delta) \supseteq \C{K}$, 
for some value of $\delta' \geq \delta$ that is not too large.
The constraint boundary could have a complex geometry, and simply maintaining a $\delta$ sized margin away from the boundary can result in big gaps between the boundary of $\C{K}$ and $\C{K}_\delta$.
For example, in two dimensions, if $\C{K}$ is polyhedral and has an acute angle, maintaining a $\delta$ margin away from both edges adjacent to the acute angle means the closest point in the $\C{K}_\delta$ to the corner may be much more than $\delta$. 
For this construction, we choose a $\BF c \in \op{relint}(\C{K})$ and a real number $r > 0$ such that $\B{B}_r^{\op{aff}(\C{K})}(\BF c) \subseteq \C{K}$.
For any $\delta < r$, we define
\begin{equation}
    \C{K}_\delta^{\BF c, r} := (1 - \frac{\delta}{r}) \C{K} + \frac{\delta}{r} \BF c. \label{eq:def:shrunkenK:cr}
\end{equation}
Clearly if $\C{K}$ is downward-closed, then so is $\C{K}_\delta^{\BF c, r}$. 
Lemma~\ref{L:shrunk_contrainst_set} shows that for any such choice of $\BF c$ and $r > 0$, we have $\frac{\delta'}{\delta} \leq \frac{D}{r}$.
See Appendix~\ref{APX:construction} for more details about the choice of $\BF c$ and $r$.
We will drop the superscripts in the rest of the paper when there is no ambiguity.
\begin{remark}
This construction is similar to the one carried out in~\cite{zhang19_onlin_contin_submod_maxim} which was for $d$-dimensional downward-closed sets. Here we impose no restrictions on $\C{K}$ beyond Assumption~\ref{assumptions}.
A simpler construction of shrunken constraint set was proposed in~\cite{chen20_black_box_submod_maxim}.
However, as we discuss in Appendix~\ref{APX:issues}, they require to be able to query outside of the constraint set.
\end{remark}

\if\TARGET0
\else
\vspace{2cm}
\fi

\subsection{Generalized DR-Submodular Frank-Wolfe}\label{sec:offline:alg}

\begin{wrapfigure}{r}{0.54\textwidth}
\vspace{-.2in}
\begin{minipage}{0.54\textwidth}
\begin{algorithm}[H]
\begin{algorithmic}[1]
\STATE {\bfseries Input:} Constraint set $\C{K}$,  iteration limit $N \geq 4$, sampling radius $\delta$, gradient step-size $\{\rho_n\}_{n=1}^N$
\STATE Construct $\C{K}_\delta$
\STATE Pick any $\z_1 \in \op{argmin}_{\z \in \C{K}_\delta} \|\z\|_\infty$
\STATE $\bar{\BF{g}}_0 \gets \BF{0}$
\FOR{$n = 1$ {\bfseries to} $N$}
\STATE $\BF g_n \gets \op{estimate-grad}(\z_n, \delta, \C{L}_0 = \op{aff}(\C{K}) - \z_1)$
\STATE $\bar{\BF{g}}_n \gets (1 - \rho_n)\bar{\BF{g}}_{n-1} + \rho_n \BF{g}_n$
\STATE $\BF v_n \gets \op{optimal-direction}(\bar{\BF g}_n, \z_n)$
\STATE $\z_{n+1} \gets \op{update}(\z_n, \BF v_n, \varepsilon)$
\ENDFOR
\STATE Output $\z_{N+1}$
\end{algorithmic}
\caption{\textsf{Generalized DR-Submodular Frank-Wolfe}}\label{ALG:main_offline}
\end{algorithm} 
\end{minipage}
\vspace{-.2in}
\end{wrapfigure}

The pseudocode of our proposed offline algorithm, \textsf{Generalized DR-Submodular Frank-Wolfe}, is shown in Algorithm 2.  At a high-level, it follows the basic template of Frank-Wolfe type methods, where over the course of a pre-specified number of iterations, the gradient (or a surrogate thereof) is calculated, an optimization sub-routine with a linear objective is solved to find a feasible point whose difference (with respect to the current solution) has the largest inner product with respect to the gradient, and then the current solution is updated to move in the direction of that feasible point.

However, there are a number of important modifications to handle properties of the objective function, constraint set, and oracle type.
For the oracle type, for instance, standard Frank-Wolfe methods assume access to a deterministic gradient oracle.
Frank-Wolfe methods are known to be sensitive to errors in estimates of the gradient (e.g., see~\cite{hassani17_gradien_method_submod_maxim}).
Thus, when only a stochastic gradient oracle or even more challenging, only a stochastic value oracle is available, the gradient estimators must be carefully designed to balance query complexity on the one hand and output error on the other. The \textsf{Black Box Gradient Estimate (BBGE)} sub-routine, presented in Algorithm 1, utilizes spherical sampling to produce an unbiased gradient estimate. This estimate is then combined with past estimates using momentum, as seen in~\cite{mokhtari20_stoch_condit_gradien_method}, to control and reduce variance.

Our algorithm design is influenced by state-of-the-art methods that have been developed for specific settings.
One of the most closely related works is \cite{chen20_black_box_submod_maxim}, which also dealt with  using value oracle access for optimizing monotone functions. They used momentum and spherical sampling techniques that are similar to the ones we used in our Algorithm 1. 
However, we modified the sampling procedure and the solution update step. 
In their work, \cite{chen20_black_box_submod_maxim} also considered a shrunken feasible region to avoid sampling close to the boundary. 
However, they assumed that the value oracle could be queried outside the feasible set (see Appendix~\ref{APX:issues} for details).

In Algorithm 3, we consider the following cases for the function and the feasible set.
\begin{enumerate}[label=(\Alph*), leftmargin=*]
\item \label{ALG:offline-monotone_dc} If $F$ is monotone  DR-submodular and $\BF{0} \in \C{K}$,
 we choose
\begin{equation*}
\op{optimal-direction}(\bar{\BF{g}}_n, \z_n) 
= \op{argmax}_{\BF{v} \in \C{K}_\delta - \z_1} \langle \BF{v}, \bar{\BF{g}}_n \rangle
,\ 
\op{update}(\z_n, \BF{v}_n, \varepsilon) = \z_n + \varepsilon \BF{v}_n,
\end{equation*}

and $\varepsilon = 1/N$.
We start at a point near the origin and always move to points that are bigger with respect to the partial order on $\B{R}^d$.
In this case, since the function is monotone, the optimal direction is a maximal point with respect to the partial order.
The choice of $\varepsilon = 1/N$ guarantees that after $N$ steps, we arrive at a convex combination of points in the feasible set and therefore the final point is also in the feasible set.
The fact that the origin is also in the feasible set shows that the intermediate points also belong to the feasible set.

\item \label{ALG:offline-non-monotone_dc} If $F$ is non-monotone DR-submodular and $\C{K}$ is a downward closed set containing $0$,
 we choose
\begin{equation*}
\op{optimal-direction}(\bar{\BF{g}}_n, \z_n) 
= \op{argmax}_{\substack{
\BF{v} \in \C{K}_\delta - \z_1 \\
\BF{v} \leq \BF{1} - \z_n
}} \langle \BF{v},\bar{\BF{g}}_n\rangle,
\ 
\op{update}(\z_n, \BF{v}_n, \varepsilon) = \z_n + \varepsilon \BF{v}_n,
\end{equation*}

and $\varepsilon = 1/N$. This case is similar to~\ref{ALG:offline-monotone_dc}.
However, since $F$ is not monotone, we need to choose the optimal direction more conservatively.

\item \label{ALG:offline-monotone_g} If $F$ is monotone DR-submodular and $\C{K}$ is a general convex set,
we choose
\begin{equation*}
\op{optimal-direction}(\bar{\BF{g}}_n, \z_n) 
= \op{argmax}_{\BF{v} \in \C{K}_\delta} \langle \BF{v},\bar{\BF{g}}_n\rangle,
\ 
\op{update}(\z_n, \BF{v}_n, \varepsilon) = (1 - \varepsilon)\z_n + \varepsilon \BF{v}_n,
\end{equation*}

and 
$\varepsilon = \log(N)/2N$. In this case, if we update like in cases (A) and (B), we do not have any guarantees of ending up in the feasible set, so we choose the update function to be a convex combination.
Unlike~\ref{ALG:offline-non-monotone_dc}, we do not need to limit ourselves in choosing the optimal direction and we simply choose $\varepsilon$ to obtain the best approximation coefficient.

\item \label{ALG:offline-non-monotone_g} If $F$ is non-monotone DR-submodular and $\C{K}$ is a general convex set, 
we choose
\begin{equation*}
\op{optimal-direction}(\bar{\BF{g}}_n, \z_n) 
= \op{argmax}_{\BF{v}\in \C{K}_\delta} \langle \BF{v},\bar{\BF{g}}_n\rangle,
\quad
\op{update}(\z_n, \BF{v}_n, \varepsilon) = (1 - \varepsilon)\z_n + \varepsilon \BF{v}_n,
\end{equation*}

and $\varepsilon = \log(2)/N$. This case is similar to~\ref{ALG:offline-monotone_g} and we choose $\varepsilon$ to obtain the best approximation coefficient.
\end{enumerate}

The choice of subroutine $\op{estimate-grad}$ and $\rho_n$ depend on the oracle.
If we have access to a gradient oracle $\hat{G}$, we set $\op{estimate-grad}(\z, \delta, \C{L}_0)$ to be the average of $B$ evaluations of $P_{\C{L}_0}(\hat{G}(\z))$.
Otherwise, we run Algorithm 1 with input $\z$, $\delta$, $\C{L}_0$.
If we have access to a deterministic gradient oracle, then there is no need to use any momentum and we set $\rho_n = 1$.
In other cases, we choose $\rho_n = \frac{2}{(n+3)^{2/3}}$.

\subsection{Approximation Guarantees for the Proposed Offline Algorithm}\label{sec:offline:analysis}

\begin{theorem}\label{T:main_offline}
{Suppose \cref{assumptions} holds.}
Let $N \geq 4$, $B \geq 1$ and choose $\BF{c} \in \C{K}$ and $r > 0$  according to Section~\ref{SS:construction}.
If we have access to a gradient oracle, we choose $\delta = 0$, otherwise we choose $\delta \in (0, r/2)$.
Then the following results hold for the output $\z_{N+1}$ of Algorithm 2.
\begin{enumerate}[label=(\Alph*), leftmargin=*]
\item \label{T:offline-monotone_dc} If $F$ is monotone DR-submodular and $\BF{0} \in \C{K}$,
then 
\begin{equation}
(1-e^{-1}) F(\z^*) - \B{E}[F(\z_{N+1})] 
\leq \frac{3 D Q^{1/2}}{N^{1/3}} + \frac{L D^2}{2 N} 
    + \delta G(2 + \frac{\sqrt{d} + D}{r}).
\end{equation}

\item \label{T:offline-non-monotone_dc} If $F$ is non-monotone DR-submodular and $\C{K}$ is a downward closed set containing $\BF{0}$,
then
\begin{equation}
e^{-1} F(\z^*) - \B{E}[F(\z_{N+1})]
\leq \frac{3 D Q^{1/2}}{N^{1/3}} + \frac{L D^2}{2 N} 
    + \delta G(2 + \frac{\sqrt{d} + 2 D}{r}).
\end{equation}

\item \label{T:offline-monotone_g} If $F$ is monotone DR-submodular and $\C{K}$ is a general convex set,
then
\begin{equation}
\frac{1}{2} F(\z^*) - \B{E}[F(\z_{N+1})] 
\leq \frac{3 D Q^{1/2} \log(N)}{2 N^{1/3}}
    + \frac{4DG + L D^2 \log(N)^2}{8 N}
    + \delta G(2 + \frac{D}{r}).
\end{equation}

\item \label{T:offline-non-monotone_g} If $F$ is non-monotone DR-submodular and $\C{K}$ is a general convex set, 
then
\begin{equation}
\frac{1}{4} (1 - \|\z_1\|_\infty) F(\z^*) - \B{E}[F(\z_{N+1})] 
\leq \frac{3 D Q^{1/2}}{N^{1/3}} + \frac{D G + 2 L D^2}{4 N} 
    + \delta G(2 + \frac{D}{r}).
\end{equation}
\end{enumerate}
In all these cases, we have
\[
Q = \begin{cases}
0
    &\text{det. grad. oracle},\\
\max \{4^{2/3}G^2,6 L^2 D^2+\frac{4 \sigma_1^2}{B} \}
    &\text{stoch. grad. oracle with variance } \sigma_1^2 > 0,\\
\max \{4^{2/3}G^2,6 L^2 D^2+\frac{4 C k G^2+2 k^2\sigma_0^2/\delta^2}{B} \}
    &\text{value oracle with variance } \sigma_0^2 \geq 0,\\
\end{cases}
\]
$C$ is a constant, $k = \op{dim}(\C{K})$, $D = \op{diam}(\C{K})$, and $\z^*$ is the global maximizer of $F$ on $\C{K}$.
\end{theorem}

\cref{T:main_offline} characterizes the  worst-case approximation ratio $\alpha$ and additive error bounds for different properties of the function and feasible region, where the additive error bounds depend on selected parameters $N$ for the number of iterations, batch size $B$, and sampling radius $\delta$.

The proof of Parts~\ref{T:offline-monotone_dc}-\ref{T:offline-non-monotone_g} is provided in Appendix~\ref{APX:monotone_dc}-\ref{APX:non-monotone_g}, respectively.

{
The proof of parts~\ref{T:offline-monotone_dc},~\ref{T:offline-non-monotone_dc} and~\ref{T:offline-non-monotone_g}, when we have access to an exact gradient oracle is similar to the proofs presented in~\cite{bian17_guaran_non_optim,bian17_contin_dr_maxim,mualem22_resol_approx_offlin_onlin_non}, respectively.
Part~\ref{T:offline-monotone_g} is the first analysis of a Frank-Wolfe type algorithm over general convex sets when the oracle can only be queried within the feasible set.
When we have access to a stochastic gradient oracle, directly using a gradient {sample} can result in arbitrary bad performance as shown in Appendix B of~\cite{hassani17_gradien_method_submod_maxim}.
The momentum technique, first used in continuous submodular maximization in~\cite{mokhtari20_stoch_condit_gradien_method}, is used when we have access to a stochastic gradient oracle.
The control on the estimate of the gradient is deferred to Lemma~\ref{L:derivative_estimate_control}.
Since the momentum technique is robust to noise in the gradient, when we only have access to a value oracle, we can use Algorithm~\ref{ALG:gradient_estimate}, similar to~\cite{chen20_black_box_submod_maxim}, to obtain an unbiased estimate of the gradient and complete the proof.
}

\cref{T:offline-complexity} converts those bounds to characterize the oracle complexity for a user-specified additive error tolerance $\epsilon$ based on oracle properties (deterministic/stochastic gradient/value). 
The 16 combinations of the problem settings listed in Table 1 are enumerated by four cases (A)--(D)  in \cref{T:main_offline} of function and feasible region properties (resulting in different approximation ratios) and the four cases 1--4 enumerated in \cref{T:offline-complexity} below of oracle properties. For the oracle properties, we consider the four cases as (Case 1): deterministic gradient oracle, (Case 2): stochastic gradient oracle, (Case 3): deterministic value oracle, and (Case 4): stochastic value oracle.

\begin{theorem}\label{T:offline-complexity}
The number of oracle calls for different oracles to achieve an $\alpha$-approximation error of smaller than $\epsilon$ using Algorithm 1 is 
\begin{equation}
\text{\bf Case 1: } \tilde{O} (1/\epsilon),\quad \text{\bf Cases 2, 3: } \tilde{O}(1/\epsilon^3), \quad \text{\bf Case 4: } \tilde{O}(1/\epsilon^5).\nonumber
\end{equation}
Moreover, in all of the cases above, if $F$ is non-monotone or $0 \in \C{K}$, we may replace $\tilde{O}$ with $O$.
\end{theorem}

See Appendix~\ref{APX:offline-complexity} for proof.

\section{Online DR-submodular optimization under bandit or semi-bandit feedback}

In this section, we first describe the Black-box Explore-Then-Commit algorithm that uses the offline algorithm for exploration, and uses the solution of the offline algorithm for exploitation. This is followed by the regret analysis of the proposed algorithm.  
This is the first algorithm for stochastic continuous DR-submodular maximization under bandit feedback and obtains state-of-the-art for semi-bandit feedback.

\subsection{Problem Setup}\label{onlinesetup}

There are typically two settings considered in online optimization with bandit feedback. 
The first is the adversarial setting, where the environment chooses a sequence of functions $F_1, \cdots, F_N$ and in each iteration $n$, the agent chooses a point $\z_n$ in the feasible set $\C{K}$, observes $F_n(z_n)$
and receives the reward $F_n(\z_n)$.
The goal is to choose the sequence of actions that minimize the following notion of expected $\alpha$-regret.
\begin{align}\label{EQ:regret-adv}
\C{R}_{\op{adv}}
:= \alpha \max_{\z \in \C{K}}\sum_{n = 1}^N F_n(\z) 
    - \B{E}\left[ \sum_{n = 1}^N F_n(\z_n) \right].
\end{align}
In other words, the agent's cumulative reward is being compared to $\alpha$ times the reward of the best \textit{constant} action in hindsight.
Note that, in this case, the randomness is over the actions of the policy.

The second is the stochastic setting, where the environment chooses a function $F: \C{K} \to \B{R}$ and a stochastic value oracle $\hat{F}$.
In each iteration $n$, the agent chooses a point $\z_n$ in the feasible set $\C{K}$, receives the reward $(\hat{F}(\z_n))_n$ by querying the oracle at $z_n$ and observes this reward.
Here the outer subscript $n$ indicates that the result of querying the oracle at time $n$, since the oracle is stochastic.
The goal is to choose the sequence of actions that minimize the following notion of expected $\alpha$-regret.
\begin{align}\label{EQ:regret-stoch}
\C{R}_{\op{stoch}}
:= \alpha N \max_{\z \in \C{K}} F(\z) 
    - \B{E}\left[ \sum_{n = 1}^N (\hat{F}(\z_n))_n \right]
= \alpha N \max_{\z \in \C{K}} F(\z) 
    - \B{E}\left[ \sum_{n = 1}^N F(\z_n) \right]
\end{align}

Further, two feedback models are considered -- bandit and semi-bandit feedback. 
In the bandit feedback setting, the agent only observes the value of the function $F_n$ at the point $\z_n$.
In the semi-bandit setting, the agent has access to a gradient oracle instead of a value oracle and observes $\hat{G}(\z_n)$ at the point $\z_n$, where $\hat{G}$ is an unbiased estimator of $\nabla F$.

In unstructured multi-armed bandit problems, any regret bound for the adversarial setup could be translated into bounds for the stochastic setup. 
However, having a non-trivial correlation between the actions of different arms complicates the relation between the stochastic and adversarial settings.
Even in linear bandits, the relation between adversarial linear bandits and stochastic linear bandits is not trivial. (e.g. see Section~29 in~\cite{lattimore2020bandit})
While it is intuitively reasonable to assume that the optimal regret bounds for the stochastic case are better than that of the adversarial case, such a result is not yet proven for DR-submodular functions. Thus, while the cases of bandit feedback has been studied in the adversarial setup, the results do not reduce to stochastic setup. We also note that in the cases where there are adversarial setup results, this paper finds that the results in the stochastic setup achieve improved regret bounds (See Table 3 in Supplementary for the comparison). 

\subsection{Algorithm for DR-submodular maximization with Bandit Feedback}

\setcounter{algorithm}{2}
\begin{wrapfigure}{r}{0.45\textwidth}
\vspace{-.2in}
\begin{minipage}{0.45\textwidth}
\begin{algorithm}[H]
\begin{algorithmic}[1]
\STATE {\bfseries Input: Horizon $T$, inner time horizon $T_0$} 
\STATE Run Algorithm~2 for $T_0$, with according to parameters described in Theorem~\ref{T:offline-complexity}.
\FOR{remaining time}
\STATE Repeat the last action of Algorithm~\ref{ALG:main_offline}.
\ENDFOR
\end{algorithmic}
\caption{\textsf{DR-Submodular Explore-Then-Commit}}
\label{ALG:DRETC}
\end{algorithm}
\end{minipage}
\end{wrapfigure}

The proposed algorithm is described in Algorithm~\ref{ALG:DRETC}. 
In Algorithm~\ref{ALG:DRETC}, if there is semi-bandit feedback in the form of a stochastic gradient sample for each action $\z_n$, we run the offline algorithm (Algorithm~\ref{ALG:main_offline}) with parameters from the proof of  case~2 of \cref{T:offline-complexity} for $T_0 = \lceil T^{3/4} \rceil$ total queries.
If only the stochastic reward for each action $\z_n$ is available (bandit feedback), we run the offline algorithm (Algorithm~\ref{ALG:main_offline}) with parameters from the proof of  case~4 of \cref{T:offline-complexity} for $T_0 = \lceil T^{5/6} \rceil$ total queries.  Then, for the remaining time (exploitation phase), we run the last action in the exploration phase.

\subsection{Regret Analysis for DR-submodular maximization with Bandit Feedback}

In this section, we provide the regret analysis for the proposed algorithm. 
We note that by \cref{T:offline-complexity},  Algorithm~\ref{ALG:main_offline} requires a sample complexity of $\tilde{O}(1/\epsilon^5)$ with a stochastic value oracle for offline problems (any of (A)--(D) in \cref{T:main_offline}).  Thus, the parameters and the results with bandit feedback are the same for all the four setups (A)--(D). Likewise,  when a stochastic gradient oracle is available, Algorithm~\ref{ALG:main_offline} requires a sample complexity of $\tilde{O}(1/\epsilon^3)$. 
Based on these sample complexities, the overall regret of online DR-submodular maximization problem is given as follows. 

\begin{theorem}\label{T:offline-to-online}
For an online constrained DR-submodular maximization problem  over a horizon $T$, where the expected reward function $F$, feasible region type $\C{K}$, and approximation ratio $\alpha$ correspond to any of the four cases (A)--(D) in \cref{T:main_offline},  \cref{ALG:DRETC} achieves $\alpha$-regret \eqref{EQ:regret-stoch} that is upper-bounded as: 
\begin{equation*}
\text{\bf Semi-bandit Feedback (Case 2): } \tilde{O}(T^{3/4}),\qquad  \text{\bf Bandit Feedback (Case 4): } \tilde{O}(T^{5/6}).
\end{equation*}
Moreover, in either type of feedback, if $F$ is non-monotone or $\BF{0} \in \C{K}$, we may replace $\tilde{O}$ with $O$.
\end{theorem}

See Appendix~\ref{APX:offline-to-online} for the proof.

\section{Conclusion}

This work provides a novel and unified approach for maximizing continuous DR-submodular functions across various assumptions on function, constraint set, and oracle access types. The proposed Frank-Wolfe based algorithm improves upon existing results for {nine} out of the {sixteen} cases considered, and presents  new results for offline DR-submodular maximization with only a value oracle. Moreover, this work presents the first regret analysis with bandit feedback for stochastic DR-submodular maximization, covering both monotone and non-monotone functions.  These contributions significantly advance the field of DR-submodular optimization (with multiple applications) and open up new avenues for future research in this area.

{\bf Limitations: } While the  number of function evaluations in the different setups considered in the paper are state of the art, lower bounds have not been investigated. 

\if\TARGET0
    \bibliographystyle{plain}
\else
    \bibliographystyle{apalike}
\fi
\bibliography{shared/references}

\newpage
\appendix
\section{Details of Related Works}\label{sec:related}

\subsection{Offline DR-submodular maximization}\label{sec:related:offline}

The authors of \cite{bian17_guaran_non_optim} considered the problem of maximizing a monotone DR-submodular function over a downward-closed convex set given a deterministic gradient oracle. They showed that a variant of the Frank-Wolfe algorithm guarantees an optimal $(1-\frac{1}{e})$-approximation for this problem. While they only claimed their result for downward-closed convex sets, their result holds under a more general setting where the convex set contains the origin. 
In \cite{bian17_contin_dr_maxim}, a non-monotone variant of the algorithm for downward-closed convex sets with $\frac{1}{e}$-approximation was proposed.

The authors of \cite{hassani17_gradien_method_submod_maxim} used gradient ascent to obtain $\frac{1}{2}$-guarantees for the maximization of a monotone DR-submodular function over a general convex set given a gradient oracle which could be stochastic. They proved that gradient ascent cannot guarantee better than a $\frac{1}{2}$-approximation by constructing a convex set $\C{K}$ and a function $F : \C{K} \to \B{R}$ such that $F$ has a local maximum that is a $\frac{1}{2}$-approximation of its optimal value on $\C{K}$.
They also showed that a Frank-Wolfe type algorithm similar to~\cite{bian17_guaran_non_optim} cannot be directly used when we only have access to a stochastic gradient oracle.
\cite{zhang22_stoch_contin_submod_maxim} extended projected gradient ascent using a line integral method, referred to as boosting, to obtain $(1-1/e)$-approximation for convex sets containing the origin.
Later, \cite{mokhtari20_stoch_condit_gradien_method} resolved the issue of stochastic gradient oracles with a momentum technique and obtained $(1-\frac{1}{e})$-approximation in the case of monotone functions over sets that contain the origin, and $\frac{1}{e}$-approximation in the case of non-monotone functions over downward closed sets. 
In~\cite{zhang22_stoch_contin_submod_maxim} and the first case in~\cite{mokhtari20_stoch_condit_gradien_method}, while they consider monotone DR-submodular functions over general convex sets $\C{K}$, they query the oracle over the convex hull of $\C{K} \cup \{\BF{0}\}$ (See Appendix~\ref{APX:constraint_vs_query}).

For non-monotone maps over general convex sets, no constant approximation ratio can be guaranteed in sub-exponential time due to a hardness result by~\cite{vondrak13}. However, \cite{durr19_nonmonotone} bypassed this issue by finding an approximation guarantee that depends on the geometry of the convex set. Specifically, they showed that given a deterministic gradient oracle for a non-monotone function over a general convex set $\C{K} \subseteq [0, 1]^d$, their proposed algorithm obtains $\frac{1}{3\sqrt{3}}(1 - h)$-approximation of the optimal value where $h := \min_{\z \in \C{K}} \| \z \|_\infty$. An improved sub-exponential algorithm was proposed by~\cite{du2022improved} that obtained a $\frac{1}{4}(1 - h)$-approximation guarantees, which is optimal. Later, \cite{du22_lyapun} provided the first polynomial time algorithm for this setting with the same approximation coefficient.

\begin{remark}
{
In the special case of maximizing a non-monotone continuous DR-submodular over a box, i.e. $[0, 1]^d$, one could discretize the problem and use discrete algorithms to solve the continuous version.
The technique has been employed in~\cite{bian17_contin_dr_maxim} to obtain  a  $\frac{1}{3}$-approximation and in~\cite{bian2019optimal,niazadeh20_optim_algor_contin_non_submod} to obtain $\frac{1}{2}$-approximations for the optimal value.
We have not included these results in Table~\ref{TBL:offline} since using discretization has only been successfully applied to the case where the convex set is a box and can not be directly used in more general settings.
}
\end{remark}

\subsection{Online DR-submodular maximization with bandit feedback}

There has been growing interest in online DR-submodular maximization in the recent years \cite{chen18_onlin_contin_submod_maxim}, \cite{chen18_projec_free_onlin_optim_stoch_gradien}, \cite{zhang19_onlin_contin_submod_maxim}, \cite{thang21_onlin_non_monot_dr_maxim}, \cite{niazadeh21_onlin_learn_offlin_greed_algor}, \cite{zhang23_onlin_learn_non_submod_maxim}, \cite{fazel22_fast_first_order_method_monot},\cite{mualem22_resol_approx_offlin_onlin_non}.
Most of these results are focused on adversarial online full-information feedback.
In the adversarial setting, the environment chooses a sequence of functions $F_1, \cdots, F_N$ and in each iteration $n$, the agent chooses a point $z_n$ in the feasible set $\C{K}$, observes $F_n$ and receives the reward $F_n(z_n)$. For the regret bound, the agents reward is being compared to $\alpha$ times the reward of the best \textit{constant} action in hindsight. With full-information feedback, if at each iteration when the agent observes $F_n$, it may be allowed to query the value of $\nabla F_n$ or maybe $F_n$ at any number of arbitrary points within the feasible set.  Further, we consider stochastic setting, where the environment chooses a function $F: \C{K} \to \B{R}$ and a sequence of independent noise functions $\eta_n : \C{K} \to \B{R}$ with zero mean.
In each iteration $n$, the agent chooses a point $\z_n$ in the feasible set $\C{K}$, receives the reward $(F + \eta_n)(\z_n)$ and observes the reward. For the regret bound, the agents reward is being compared to $\alpha$ times the reward of the best action. Detailed formulation of adversarial and stochastic setups and why adversarial results cannot be reduced to stochastic results is given in Section \ref{onlinesetup}. In this paper, we consider two feedback models -- bandit feedback where only the (stochastic) reward value is available and semi-bandit feedback where a single stochastic sample of the gradient at the location is provided.

\begin{table}[H]
\begin{center}
\begin{tabular}{ |  c |  c | c | c | c | c | }
\hline
Function & Set &  Setting & Reference & Appx.  &  Regret \\
\hline
\multirow{6}*{ {Monotone} }
& \multirow{4}*{$0 \in \C{K}$}
&   \multirow{1}*{stoch.}
      & {\color{blue}This paper}
        & $1-1/e$ &  $O(T^{5/6})$ \\
    \cline{3-6}
&  & \multirow{3}*{adv.}
      & \cite{zhang19_onlin_contin_submod_maxim}
        & $1-1/e$ &  $O(T^{8/9})$ \\
&  &  & \cite{niazadeh21_onlin_learn_offlin_greed_algor}
        & $1-1/e$ &  $O(T^{5/6})$ \\
&  &  & \cite{wan23_bandit_multi_dr_submod_maxim}$\dagger$
        & $1-1/e$ &  $O(T^{3/4})$ \\
\cline{2-6}
& \multirow{2}*{{general}}
    & \multirow{1}*{stoch.}
      & {\color{blue}This paper}
        & $1/2$   & $\tilde{O}(T^{5/6})$ \\
    \cline{3-6}
& &  \multirow{1}*{adv.}
      & - &  & \\
\hline
\multirow{4}*{ {Non-monotone} }
&  \multirow{2}*{{d.c.}}
    & \multirow{1}*{stoch.}
      & {\color{blue}This paper}
        & $1/e$   &  $O(T^{5/6})$ \\
    \cline{3-6}
& &  \multirow{1}*{adv.}
      &  \cite{zhang23_onlin_learn_non_submod_maxim}
        & $1/e$   &  $O(T^{8/9})$ \\
\cline{2-6}
& \multirow{2}*{{general}}
    & \multirow{1}*{stoch.}
      & {\color{blue}This paper}
        & $\frac{1-h}{4}$   &  $O(T^{5/6})$ \\
    \cline{3-6}
 & & \multirow{1}*{adv.}
      & - & & \\
\hline
\end{tabular}
\end{center}
\caption{This table presents the different results for the regret for DR-submodular maximization under bandit feedback, and gives the related works and regret bounds in the adversarial case.
Note that the result marked by $\dagger$ uses a convex optimization subroutine at each iteration which could be even more computationally expensive than projection.
As before, we have $h := \min_{\x \in \C{K}} \|x\|_\infty$.
} 
\end{table}

{\bf Bandit Feedback:} We note that this paper is the first work for bandit feedback for stochastic online DR-submodular maximization. The prior works on this topic has been in the adversarial setup \cite{zhang19_onlin_contin_submod_maxim,zhang23_onlin_learn_non_submod_maxim,niazadeh21_onlin_learn_offlin_greed_algor,wan23_bandit_multi_dr_submod_maxim}, and the results in this work is compared with their results in Table 3. In~\cite{zhang19_onlin_contin_submod_maxim}, the adversarial online setting with bandit feedback has been studied for monotone DR-submodular functions over downward-closed convex sets.
Later~\cite{zhang23_onlin_learn_non_submod_maxim} extended this framework to the setting with non-monotone DR-submodular functions over downward-closed convex sets.
\cite{niazadeh21_onlin_learn_offlin_greed_algor} described a framework for converting certain greedy-type offline algorithms with robustness guarantees into adversarial online algorithms for both full-information and bandit feedback.
They apply their framework to obtain algorithms for non-monotone functions over a box, with $\frac{1}{2}$-regret of $\tilde{O}(T^{4/5})$, and monotone function over downward-closed convex sets.
The offline algorithm they use for downward-closed convex sets is the one described in~\cite{bian17_guaran_non_optim} which only requires the convex set to contain the origin.
They also use the construction of the shrunk constraint set described in~\cite{zhang19_onlin_contin_submod_maxim}. 
By replacing that construction with ours, the result of~\cite{niazadeh21_onlin_learn_offlin_greed_algor} could be extended to monotone functions over all convex sets containing the origin.
\cite{wan23_bandit_multi_dr_submod_maxim} improved the regret bound for monotone functions over convex sets containing the origin to $O(T^{3/4})$.
However, they use a convex optimization subroutine at each iteration which could be even more computationally expensive than projection.
 
{\bf Semi-bandit Feedback:} In semi-bandit feedback, a single stochastic sample of the gradient is available. The problem has been considered in \cite{chen18_projec_free_onlin_optim_stoch_gradien}, while the results have an error (See Appendix~\ref{APX:issues}). Further, they  only obtain $\frac{1}{e}$-regret for the monotone case. One could consider a generalization of the adversarial and stochastic setting in the following manner.
The environment chooses a sequence of functions $F_n$ and a sequence of value oracles $\hat{F}_n$ such that $\hat{F}_n$ estimates $F_n$.
In each iteration $n$, the agent chooses a point $\z_n$ in the feasible set $\C{K}$, receives the reward $(\hat{F}_n(\z_n))_n$ by querying the oracle at $z_n$ and observes this reward.
The goal is to choose the sequence of actions that minimize the following notion of expected $\alpha$-regret.
\begin{align}\label{EQ:regret-stoch-adv}
\C{R}_{\op{stoch-adv}}
&:= \alpha \max_{\z \in \C{K}} \sum_{n = 1}^N F_n(\z) 
    - \B{E}\left[ \sum_{n = 1}^N (\hat{F}_n(\z_n))_n \right] \nonumber \\
&= \alpha \max_{\z \in \C{K}} \sum_{n = 1}^N F_n(\z) 
    - \B{E}\left[ \sum_{n = 1}^N F_n(\z_n) \right]
\end{align}

Algorithm~3 in~\cite{chen18_onlin_contin_submod_maxim} solves this problem in semi-bandit feedback setting with a deterministic value oracle and stochastic gradient oracles.
Any bound for a problem in this setting implies bounds for stochastic semi-bandit and adversarial semi-bandit settings.
The same is true for Mono-Frank-Wolfe Algorithms in ~\cite{zhang19_onlin_contin_submod_maxim,zhang23_onlin_learn_non_submod_maxim}.
We have included these results in Table~\ref{TBL:online:short} as benchmark to compare with results in stochastic setting.

\section{Constraint Set and Query Set}\label{APX:constraint_vs_query}

In this work, we made the assumption that the query set is identical to the constraint set, i.e. oracles can only be queried within the constraint set.
To the best of our knowledge, except in the context of online optimization with (semi-)bandit feedback, this is the first work on DR-submodular maximization that explicitly considers this assumption.
Previous works assumed that we may query the oracle at any point within the unit box $[0, 1]^d$.
Algorithms designed for non-monotone functions in prior works already satisfied the assumption we consider, so no changes in algorithms, proofs, or results are needed. 
However, the situation is different when the function is monotone.
This assumption allows us to explain a previously unexplained gap in approximation guarantees for monotone DR-submodular maximization.
Specifically, some prior works (enumerated below) studying monotone DR-submodular maximization over general convex sets obtained approximation guarantees of $1/2$ while others obtained $1-1/e$.

First we describe how some of previous results in literature with no apparent restriction on the query set may be reformulated as problems where the query set is equal to the constraint set.
Let $\C{K} \subseteq [0, 1]^d$ be a convex set, and define $\C{K}^*$ as the convex hull of $\C{K} \cup \{\BF{0}\}$.
For a problem in the setting of monotone functions over a general set $\C{K}$, we can consider the same problem on $\C{K}^*$. 
Since the function is monotone, the optimal solution in $\C{K}^*$ is the same as the optimal solution in $\C{K}$. 
However, solving this problem in $\C{K}^*$ may require evaluating the function in the larger set $\C{K}^*$, which may not always be possible. 
In fact, the result of~\cite{mokhtari20_stoch_condit_gradien_method} and~\cite{zhang22_stoch_contin_submod_maxim} mentioned in Table 1 are for monotone functions over general convex sets $\C{K}$, but their algorithms require evaluating the function on $\C{K}^*$. 
This is why we have classified their results as algorithms for convex sets that contain the origin. 
The problem of offline DR-submodular maximization with only a value oracle was first considered by~\cite{chen20_black_box_submod_maxim} for monotone maps over convex sets that contain the origin. 
However, their result requires querying in a neighborhood of $\C{K}^*$ which violates our requirement to only query the oracle within the feasible set (see Appendix~\ref{APX:issues}).

In~\cite{hassani17_gradien_method_submod_maxim}, a $1/2$ approximation guarantee was obtained by a projected gradient ascent method and this was shown by proving that the algorithm tends to a stationary point and proving that any stationary point is at least $1/2$ as good as the optimal point. 
Moreover, they construct examples with stationary points that are no better that $1/2$ of the optimal point.

The $1-1/e$ approximation guarantee was first reported for Frank-Wolfe methods, which (superficially) suggests that the gap may be due to algorithm or analysis differences.
Later,~\cite{zhang22_stoch_contin_submod_maxim} developed a projected gradient ascent based method that obtains a $1-1/e$ approximation guarantee where they consider general constraint set but their query set contains the origin.

However, the gap is not attributable to algorithm or analysis differences, but instead due to the fact that the query sets are different.
In other words, the results that obtain a $1-1/e$ approximation guarantee are solving a different problem than the ones obtaining a $1/2$ approximation guarantee.
A key ingredient to obtain $1-1/e$ is the ability to query the (gradient) oracle within the convex hull of $\mathcal{K}\cup\{0\}$.
For monotone submodular maximization over general convex sets (not necessarily containing the origin), we can only guarantee a coefficient of $1/2$, both for Frank-Wolfe type methods (our work) and projection based methods (i.e.~\cite{hassani17_gradien_method_submod_maxim}).
Therefore, the $1/2$ approximation could very well be optimal in its own setting. 

To the best of our knowledge, in every paper where the $1/2$ approximation coefficient and $1-1/e$ approximation coefficient in the monotone setting are compared, the comparison was (unwittingly) between problems that are inherently mathematically different: 
\cite{hassani17_gradien_method_submod_maxim} and~\cite{chen18_onlin_contin_submod_maxim} in experiments and main text; 
\cite{chen18_projec_free_onlin_optim_stoch_gradien} and~\cite{chen20_black_box_submod_maxim} in experiments; 
\cite{zhang23_onlin_learn_non_submod_maxim,mualem22_resol_approx_offlin_onlin_non}, and~\cite{durr19_nonmonotone} in related work section, \cite{mokhtari20_stoch_condit_gradien_method} in the introduction and Table 2, \cite{zhang22_stoch_contin_submod_maxim} and~\cite{fazel22_fast_first_order_method_monot} in the main claims.

\paragraph{Conjecture}
\textit{The problem of maximizing a monotone DR-submodular continuous function subject to a general convex constraint, where oracle queries are limited to the feasible region, is NP-hard. For any $\epsilon > 0$, it cannot be approximated in polynomial time to within a ratio of $1/2 + \epsilon$ (up to low-order terms), unless $RP = NP$.}

\section{Brief discussion on oracle models in applications}

For many problems, the ability to evaluate gradients directly requires strong assumptions about problem-specific parameters.
Influence maximization and profit maximization form a family of problems that model choosing advertising resource allocations to maximize the expected number of customers, where there is an underlying diffusion model for how advertising resources spent (stochastically) activate customers over a social network. 
For common diffusion models, the objective function is known to be DR-submodular (see for instance~\cite{bian17_contin_dr_maxim} or~\cite{gu2023profit}).
The revenue (expected number of activated customers) is a monotone objective function; total profit (revenue from activated customers minus advertising costs) is a non-monotone objective.
One significant challenge with these problems is that the objective function (and the gradients) cannot be analytically evaluated for general (non-bipartite) networks, even if all the underlying diffusion model parameters are known exactly. 
The mildest assumptions on knowledge/observability of the network diffusions for offline variants (respectively actions for online variants), especially fitting for user privacy and/or third-party access, leads to instantiations of queries as the agent selecting an advertising allocation within the budget (i.e., feasible point) and observing a (stochastic) count of activated customers. 
This corresponds to stochastic value oracle queries over the feasible region (respectively bandit feedback for online variants).

\section{Comments on previous results in literature}\label{APX:issues}

\paragraph{ Construction of $\C{K}'$ and error estimate in~\cite{chen20_black_box_submod_maxim} }

In~\cite{chen20_black_box_submod_maxim}, the set $\C{K}' + \delta \BF{1}$ plays a role similar to the set $\C{K}_\delta$ defined in this paper.
Algorithm~\ref{ALG:main_offline}, in the case with access to value oracle for monotone DR-submodular function with the constraint set $\C{K}$, such that $\op{aff}(\C{K}) = \B{R}^d$ and $\BF{0} \in \C{K}$, reduced to BBCG algorithm in~\cite{chen20_black_box_submod_maxim} if we replace $\C{K}_\delta$ with their construction of $\C{K}' + \delta \BF{1}$.
In their paper, $\C{K}'$ is defined by
\begin{equation}
\C{K}' := (\C{K} - \delta \BF{1}) \cap [0, 1 - 2 \delta]^d.
\end{equation}
There are a few issues with this construction and the subsequent analysis that requires more care.
\begin{enumerate}
\item \textit{The BBCG algorithm almost always needs to be able to query the value oracle outside the feasible set.}

We have
\[
\C{K}' + \delta \BF{1} = \C{K} \cap [\delta, 1 - \delta]^d.
\]
The BBCG algorithm starts at $\delta \BF{1}$ and behaves similar to Algorithm~\ref{ALG:main_offline} in the monotone $\BF{0} \in \C{K}$ case.
It follows that the set of points that BBCG requires to be able to query is
\begin{align*}
Q_\delta 
:= \B{B}_\delta( \op{convex-hull}( (\C{K}' + \delta \BF{1}) \cup \{\delta \BF{1}\}) )
= \B{B}_\delta( \op{convex-hull}(\C{K} \cup \{\delta \BF{1}\}) \cap [\delta, 1 - \delta]^d ).
\end{align*}
If $\BF{1} \in \C{K}$, then the problem becomes trivial since $F$ is monotone.
If $\C{K}$ is contained in the boundary of $[0, 1]^d$, then we need to restrict ourselves to the affine subspace containing $\C{K}$ and solve the problem in a lower dimension in order to be able to use BBCG algorithm as $\C{K}'$ will be empty otherwise.
We want to show that in all other cases, $Q_\delta \setminus \C{K} \neq \emptyset$.
If $\C{K}'$ is non-empty and $\BF{1} \notin \C{K}$, then let $\x_\delta$ be a maximizer of $\|\cdot\|_\infty$ over $\C{K}' + \delta \BF{1}$.
If $\x_\delta \neq (1 - \delta)\BF{1}$, then there is a point $\y \in \B{B}_\delta(\x_\delta) \cap [\delta, 1 - \delta]^d \subseteq Q_\delta$ such that $\y > \x$ which implies that $\y \notin \C{K}$.
Therefore, we only need to prove the statement when $(1 - \delta)\BF{1} \in \C{K} \cap [\delta, 1 - \delta]^d$ for all small $\delta$.
In this case, since $\C{K}$ is closed, we see that $(1 - \delta)\BF{1} \to \BF{1} \in \C{K}$.
In other words, except in trivial cases, BBCG always requires being able to query outside the feasible set.

\item \textit{The exact error bound could be arbitrarily far away from the correct error bound depending on the geometry of the constraint set.}

In Equation~(69) in the appendix of~\cite{chen20_black_box_submod_maxim}, it is mentioned that
\begin{equation}\label{EQ:chen_69}
\tF(\x^*_\delta) \geq \tF(\x^*) - \delta G \sqrt{d},
\end{equation}
where $\x^*$ is the optimal solution and $\x^*_\delta$ is the optimal solution within $\C{K}' + \delta \BF{1}$ and $G$ is the Lipschitz constant.
Next we construct an example where this inequality does not hold.

Consider the set $\C{K} = \{ (x,y) \in [0, 1]^2 \mid x + \lambda y \leq  1 \}$ for some value of $\lambda$ to be specified and let $F((x, y)) = Gx$.
Clearly we have $\x^* = (1, 0)$.
Thus, for any $\delta > 0$, we have
\[
\C{K}' + \delta \BF{1} = \{ (x, y) \in [\delta, 1-\delta]^2 \mid x + \lambda y \leq 1\}.
\]
It follows that when $\lambda \leq \frac{1}{\delta} - 1$, then $\C{K}'$ is non-empty and $\x^*_\delta = (1 - \lambda \delta, \delta)$.
Then we have
\[
\tF(\x^*_\delta) - \tF(\x^*) = - \lambda \delta G.
\]

Therefore, \eqref{EQ:chen_69}  is correct if and only if $\lambda \leq \sqrt{d}=\sqrt{2}$. Since this does not hold in general as $\lambda$ depends on the geometry of the convex set, this equation is not true in general making the overall proof incorrect. 
The issue here is that $\lambda$, which depends on the geometry of the convex set $\C{K}$, should appear in \eqref{EQ:chen_69}.
Without restricting ourselves to convex sets with ``controlled'' geometry and without including a term, such as $\frac{1}{r}$ in Theorem~\ref{T:main_offline}, we would not be able to use this method to obtain an error bound.
We note that while their analysis has an issue, the algorithm is still fine. 
Using a proof technique similar to ours, their proof can be fixed, more precisely, we can modify \eqref{EQ:chen_69} in a manner similar to \eqref{EQ:corrected_69_A} and~\eqref{EQ:corrected_69_C}, depending on the case, and that will help fix their proofs.
\end{enumerate}

\paragraph{One-Shot Frank-Wolfe algorithm in~\cite{chen18_projec_free_onlin_optim_stoch_gradien} }

In \cite{chen18_projec_free_onlin_optim_stoch_gradien}, the authors claim their proposed algorithm, One-Shot Frank-Wolfe (OSFW), achieves a $(1-\frac{1}{e})$-regret for monotone DR-submodular maximization under semi-bandit feedback for general convex set with oracle access to the entire domain of $F$, i.e. $[0, 1]^d$.
In their regret analysis in the last page of the supplementary material, the inequality $(1 - 1/T)^t \leq 1/e$ is used for all $0 \leq t \leq T-1$.
Such an inequality holds for $t = T$ but as $t$ decreases, the value of $(1 - 1/T)^t$ becomes closer to 1 and the inequality fails.
If we do not use this inequality and continue with the proof, we end up with the following approximation coefficient.
\begin{align*}
1 - \frac{1}{T}\sum_{t = 0}^{T-1} (1 - 1/T)^t
= 1 - \frac{1}{T} \cdot \frac{1 - (1 - 1/T)^T}{1 - (1 - 1/T)}
= 1 - (1 - (1 - 1/T)^T)
= (1 - 1/T)^T
\sim \frac{1}{e}.
\end{align*}

\section{Useful lemmas}

Here we state some lemmas from the literature that we will need in our analysis of DR-submodular functions.

\begin{lemma}[Lemma~2.2 of~\cite{mualem22_resol_approx_offlin_onlin_non}]\label{L:F_join_non-monotone}
For any two vectors $\x, \y \in [0, 1]^d$ and any continuously differentiable non-negative DR-submodular function $F$ we have
\[
F(\x \vee \y) \geq (1 - \|\x\|_\infty) F(\y).
\]
\end{lemma}

The following lemma can be traced back to~\cite{hassani17_gradien_method_submod_maxim} (see Inequality~7.5 in the arXiv version), and is also explicitly stated and proved in~\cite{durr19_nonmonotone}.

\begin{lemma}[Lemma~1 of~\cite{durr19_nonmonotone}]\label{L:F_join_meet_non-monotone}
For every two vectors $\x, \y \in [0, 1]^d$ and any continuously differentiable non-negative DR-submodular function $F$ we have
\[
\bra \nabla F(\x), \y - \x \ket
\geq 
F(\x \vee \y) + F(\x \wedge \y) - 2 F(\x).
\]
\end{lemma}

\section{Smoothing trick}

The following Lemma is well-known when $\op{aff}(\C{D}) = \B{R}^d$ (e.g., Lemma~1 in~\cite{chen20_black_box_submod_maxim}, Lemma~7 in~\cite{zhang19_onlin_contin_submod_maxim}).
The proof in the general case is similar to the special case $\op{aff}(\C{D}) = \B{R}^d$.
\begin{lemma}\label{L:smooth_approx}
If $F : \C{D} \to \B{R}$ is DR-submodular, $G$-Lipschitz continuous, and $L$-smooth, then so is $\tF_\delta$ and for any $\x \in \C{D}$ such that $\B{B}_\delta^{\op{aff}(\C{D})}(\x) \subseteq \C{D}$, we have
\[
\| \tF_\delta(\x) - F(\x) \| \le \delta G.
\]
Moreover, if $F$ is monotone, then so is $\tF_\delta$.
\end{lemma}
\begin{proof}
Let $A := \op{aff}(\C{D})$ and $A_0 := \op{aff}(\C{D}) - \x$ for some $\x \in \C{D}$.
Using the assumption that $F$ is $G$-Lipschitz continuous, we have
\begin{align*}
|\tF(\x) - \tF(\y)| 
&= \left| \B{E}_{\vv \sim \B{B}_1^{A_0}(\BF{0})} [F(\x + \delta \vv) - F(\y+\delta \vv)] \right| \\
&\leq \B{E}_{\vv \sim \B{B}_1^{A_0}(\BF{0})}[ |F(\x + \delta \vv) - F(\y + \delta \vv)|] \\
&\leq \B{E}_{\vv \sim \B{B}_1^{A_0}(\BF{0})}[G \|(\x + \delta \vv)-(\y + \delta \vv)\| ] \\
&= G \|\x - \y\|,
\end{align*}
and
\begin{align*}
|\tF(\x)-F(\x)| 
&= | \B{E}_{\vv\sim \B{B}_1^{A_0}(\BF{0})}[F(\x+\delta \vv)-F(\x)] | \\
&\leq \B{E}_{\vv\sim \B{B}_1^{A_0}(\BF{0})} [|F(\x+\delta \vv)-F(\x)|] \\
&\leq \B{E}_{\vv\sim \B{B}_1^{A_0}(\BF{0})}[G \delta \|\vv \|] \\
&\leq \delta G.
\end{align*}
If $F$ is $G$-Lipschitz continuous and continuous DR-submodular, then $F$ is differentiable and we have
$\nabla F(\x) \geq \nabla F(\y)$
for $\forall \x \leq \y$.
By definition of $\tF$, we see that $\tF$ is also differentiable and
\begin{align*}
\nabla \tF(\x) - \nabla \tF(\y) 
&= \nabla \B{E}_{\vv\sim \B{B}_1^{A_0}(\BF{0})} 
[F(\x+\delta \vv)]-  \nabla \B{E}_{\vv\sim \B{B}_1^{A_0}(\BF{0})} 
[F(\y+\delta \vv)] \\
&= \B{E}_{\vv\sim \B{B}_1^{A_0}(\BF{0})}[\nabla F(\x+\delta \vv) - \nabla F(\y+\delta \vv) ] \\
&\geq \B{E}_{\vv\sim \B{B}_1^{A_0}(\BF{0})} [0] 
= 0,
\end{align*}
for all $\x \leq \y$.

If $F$ is $L$-smooth, then we have $\| \nabla F(\x) - \nabla F(\y) \| \leq L \| \x - \y \|$, for all $\x, \y \in \C{D}$.
Therefore, we have
\begin{align*}
\| \nabla \tF(\x) - \nabla \tF(\y) \|
&= \| \nabla \B{E}_{\vv\sim \B{B}_1^{A_0}(\BF{0})} 
[F(\x+\delta \vv)]-  \nabla \B{E}_{\vv\sim \B{B}_1^{A_0}(\BF{0})} 
[F(\y+\delta \vv)] \| \\
&= \| \B{E}_{\vv\sim \B{B}_1^{A_0}(\BF{0})} 
[\nabla F(\x+\delta \vv)]-  \B{E}_{\vv\sim \B{B}_1^{A_0}(\BF{0})} 
[ \nabla F(\y+\delta \vv)] \| \\
&\leq \B{E}_{\vv\sim \B{B}_1^{A_0}(\BF{0})}[ \| \nabla F(\x+\delta \vv) - \nabla F(\y+\delta \vv) \| ] \\
&\leq \B{E}_{\vv\sim \B{B}_1^{A_0}(\BF{0})} [L \| \x - \y \|] 
= L \| \x - \y \|,
\end{align*}
for all $\x \leq \y$.

If $F$ is monotone, then we have $F(\x) \leq F(\y)$ for all $\x \leq \y$.
Therefore
\begin{align*}
\tF(\x) - \tF(\y) 
&= \B{E}_{\vv \sim \B{B}_1^{A_0}(\BF{0})}[F(\x+\delta \vv)]
    - \B{E}_{\vv \sim \B{B}_1^{A_0}(\BF{0})}[F(\y+\delta \vv)] \\
&= \B{E}_{\vv \sim \B{B}_1^{A_0}(\BF{0})}[F(\x+\delta \vv) - F(\y+\delta \vv)] \\
&\leq \B{E}_{\vv \sim \B{B}_1^{A_0}(\BF{0})}[0]
= 0,
\end{align*}
for all $\x \leq \y$.
Hence $\tF$ is also monotone.
\end{proof}

\begin{lemma}[Lemma~10 of \citep{shamir17_optim_algor_bandit_zero_order}]\label{L:tilde_F_mean_variance_old}
Let $\C{D} \subseteq \B{R}^d$ such that $\op{aff}(\C{D}) = \B{R}^d$.
Assume $F : \C{D} \to \B{R}$ is a $G$-Lipschitz continuous function and let $\tF$ be its $\delta$-smoothed version.
For any $\z \in \C{D}$ such that $\B{B}_\delta(\z) \subseteq \C{D}$, we have
\[
\B{E}_{\uu \sim S^{d-1}}\left[
    \frac{d}{2\delta}(F(\z+\delta \uu)-F(\z-\delta \uu))\uu
\right]
= \nabla \tF(\z),
\]
\[
\B{E}_{\uu \sim S^{d-1}}\left[
    \|\frac{d}{2\delta}(F(\z+\delta \uu)-F(\z-\delta \uu))\uu - \nabla \tF(\z) \|^2
\right]
\le CdG^2,
\]
where $C$ is a constant.
\end{lemma}

When the convex feasible region $\C{K}$ lies in an affine subspace, we cannot employ the standard spherical sampling method.
We extend \cref{L:tilde_F_mean_variance_old} to that case.

\begin{lemma}\label{L:tilde_F_mean_variance}
Let $\C{D} \subseteq \B{R}^d$ and $A := \op{aff}(\C{D})$.
Also let $A_0$ be the translation of $A$ that contains $0$ and let $k = \op{dim}(A)$.
Assume $F : \C{D} \to \B{R}$ is a $G$-Lipschitz continuous function and let $\tF$ be its $\delta$-smoothed version.
For any $\z \in \C{D}$ such that $\B{B}_\delta^{A}(\z) \subseteq \C{D}$, we have
\[
\B{E}_{\uu \sim S^{d-1} \cap A_0}\left[
    \frac{k}{2\delta}(F(\z+\delta \uu)-F(\z-\delta \uu))\uu
\right]
= \nabla \tF(\z),
\]
\[
\B{E}_{\uu \sim S^{d-1} \cap A_0}\left[
    \|\frac{k}{2\delta}(F(\z+\delta \uu)-F(\z-\delta \uu))\uu - \nabla \tF(\z) \|^2
\right]
\le C k G^2,
\]
where $C$ is the constant in Lemma~\ref{L:tilde_F_mean_variance_old}.
\end{lemma}

\begin{proof}
First consider the case where $A = \B{R}^k \times (0, \cdots, 0)$.
In this case, we restrict ourselves to first $k$ coordinates and see that the problem reduces to Lemma~\ref{L:tilde_F_mean_variance_old}.

For the general case, let $O$ be an orthonormal transformation that maps $\B{R}^k \times (0, \cdots, 0)$ into $A_0$.
Now define $\C{D}' = O^{-1}(\C{D} - \z)$ and $F' : \C{D}' \to \B{R} : x \mapsto F(O(x) + \z)$.
Let $\tF'$ be the $\delta$-smoothed version of $F'$.
Note that $O\left( \nabla \tF'(0) \right) = \nabla \tF(\z)$.
On the other hand, we have 
\[
\op{aff}(\C{D}') 
= O^{-1}(A - \z)
= O^{-1}(A_0)
= \B{R}^k \times (0, \cdots, 0).
\]
Therefore
\[
\B{E}_{\uu \sim S^{d-1} \cap (\B{R}^k \times (0, \cdots, 0))}\left[
    \frac{k}{2\delta}(F'(\delta \uu)-F'(-\delta \uu))\uu
\right]
= \nabla \tF'(0),
\]
and
\[
\B{E}_{\uu \sim S^{d-1} \cap (\B{R}^k \times (0, \cdots, 0)}\left[
    \|\frac{k}{2\delta}(F'(\delta \uu)-F'(-\delta \uu))\uu - \nabla \tF'(0) \|^2
\right]
\le C k G^2.
\]
Hence, if we set $\vv = O^{-1}(\uu)$, we have
\begin{align*}
\B{E}_{\uu \sim S^{d-1} \cap A_0}&\left[
    \frac{k}{2\delta}(F(\z+\delta \uu) - F(\z-\delta \uu))\uu
\right] \\
&= \B{E}_{\vv \sim S^{d-1} \cap (\B{R}^k \times (0, \cdots, 0))}\left[
    \frac{k}{2\delta}(F'(\delta \vv) - F'(-\delta \vv)) O(\vv)
\right] \\
&= O\left( \B{E}_{\vv \sim S^{d-1} \cap (\B{R}^k \times (0, \cdots, 0))}\left[
    \frac{k}{2\delta}(F'(\delta \vv) - F'(-\delta \vv))\vv
\right] \right) \\
&= O\left( \nabla \tF'(0) \right) \\
&= \nabla \tF(\z).
\end{align*}
Similarly
\begin{align*}
\B{E}_{\uu \sim S^{d-1} \cap A_0} &\left[
    \|\frac{k}{2\delta}(F(\z+\delta \uu)-F(\z-\delta \uu))\uu - \nabla \tF(\z) \|^2
\right] \\
&= \B{E}_{\vv \sim S^{d-1} \cap (\B{R}^k \times (0, \cdots, 0))} \left[
    \|\frac{k}{2\delta}(F'(\delta \vv)-F'(-\delta \vv))O(\vv) 
        - O\left( \nabla \tF'(0) \right) \|^2
\right] \\
&= \B{E}_{\vv \sim S^{d-1} \cap (\B{R}^k \times (0, \cdots, 0))} \left[
    \|O\left( \frac{k}{2\delta}(F'(\delta \vv)-F'(-\delta \vv))\vv 
        - \nabla \tF'(0) \right) \|^2
\right] \\
&= \B{E}_{\vv \sim S^{d-1} \cap (\B{R}^k \times (0, \cdots, 0))} \left[
    \|\frac{k}{2\delta}(F'(\delta \vv)-F'(-\delta \vv))\vv 
        - \nabla \tF'(0) \|^2
\right] \\
&\leq C k G^2.
\qedhere
\end{align*}
\end{proof}

\begin{remark}
Note that the same argument may be applied to obtain the one-point gradient estimator:
\[
\B{E}_{\uu \sim S^{d-1} \cap A_0}\left[
    \frac{k}{\delta} F(\z+\delta \uu) \uu
\right]
= \nabla \tF(\z).
\]
\end{remark}

\section{Construction of \texorpdfstring{$\C{K}_\delta$}{Shrunk Constraint Set}}\label{APX:construction}

\begin{lemma}\label{L:norm_z_1}
Let $\C{K} \subseteq [0, 1]^d$ be a convex set containing the origin.
Then for any choice of $\BF c$ and $r$ with $\B{B}_r^{\op{aff}(\C{K})}(\BF c) \subseteq \C{K}$, we have
\[
\argmin_{\z \in \C{K}_\delta} \|\z\|_\infty = \frac{\delta}{r} \BF c
\quad \text{and} \quad
\min_{\z \in \C{K}_\delta} \|\z\|_\infty \leq \frac{\delta}{r}.
\]
\end{lemma}

\begin{proof}
The claim follows immediately from the definition and the fact that $\|\BF c\|_\infty\leq 1$.
\end{proof} 

\begin{lemma}\label{L:shrunk_contrainst_set}
Let $\C{K}$ be an arbitrary convex set, $D := \op{Diam}(\C{K})$ and $\delta' := \frac{\delta D}{r}$.
We have
\[
\B{B}_\delta^{\op{aff}(\C{K})}(\C{K}_\delta) 
\subseteq 
\C{K}
\subseteq
\B{B}_{\delta'}^{\op{aff}(\C{K})}(\C{K}_\delta).
\]
\end{lemma}
\begin{proof}
Define $\psi : \C{K} \to \C{K}_\delta := \x \mapsto (1 - \frac{\delta}{r})\x + \frac{\delta}{r} \BF c$.
Let $\y \in \C{K}_\delta$ and $\x = \psi^{-1}(\y)$.
Then
\begin{align*}
\B{B}_\delta^{\op{aff}(\C{K})}(\y)
&= \B{B}_\delta^{\op{aff}(\C{K})}(\psi(\x))
= \B{B}_\delta^{\op{aff}(\C{K})}((1 - \frac{\delta}{r})\x + \frac{\delta}{r} \BF c) \\
&= (1 - \frac{\delta}{r})\x + \B{B}_\delta^{\op{aff}(\C{K})}(\frac{\delta}{r} \BF c)
= (1 - \frac{\delta}{r})\x + \frac{\delta}{r} \B{B}_r^{\op{aff}(\C{K})}(\BF c)
\subseteq \C{K},
\end{align*}
where the last inclusion follows from the fact that $\C{K}$ is convex and contains both $\x$ and $\B{B}_r^{\op{aff}(\C{K})}(\BF c)$.
On the other hand, for any $\x \in \C{K} \subseteq \op{aff}(\C{K})$, we have
\begin{align*}
\| \psi(\x) - \x \|
= \frac{\delta}{r} \| \x - \BF c \| 
< \frac{\delta}{r} D = \delta'.
\end{align*}
Therefore 
\[
\x \in 
\B{B}_{\delta'}(\psi(\x)) \cap \op{aff}(\C{K})
=
\B{B}_{\delta'}^{\op{aff}(\C{K})}(\psi(\x))
\subseteq 
\B{B}_{\delta'}^{\op{aff}(\C{K})}(\C{K}_\delta).
\qedhere
\]
\end{proof}

\paragraph{Choice of $\BF{c}$ and $r$}
While the results hold for any choice of $\BF{c} \in \C{K}$ and $r$ with $\B{B}_r^{\op{aff}}(\BF{c}) \subseteq \C{K}$, as can be seen in Theorem~\ref{ALG:main_offline}, the approximation errors depends linearly on $1/r$.
Therefore, it is natural to choose the point $\BF{c}$ that maximizes the value of $r$, the \emph{Chebyshev center} of $\C{K}$. 

\paragraph{Analytic Constraint Model --- Polytope} When the feasible region $\C{K}$ is characterized by a set of $q$ linear constraints $\BF{A} \x \leq \BF{b}$ with a known coefficient matrix $\BF{A} \in \B{R}^{q \times d}$ and vector $\BF{b} \in\B{R}^{q}$, thus $\C{K}$ is a polytope, by the linearity of the transformation \eqref{eq:def:shrunkenK:cr}, the shrunken feasible region $\C{K}_\delta$ is similarly characterized by a (translated) set of $q$ linear constraints $\BF{A} \x \leq (1-\frac{\delta}{r})\BF{b} + \frac{\delta}{r}\BF{A}\BF{c}$.

\section{Variance reduction via momentum}

In order to prove main regret bounds, we need the following variance reduction lemma, which is crucial in characterizing how much the variance of the gradient estimator can be reduced by using momentum.
This lemma appears in~\citep{chen18_projec_free_onlin_optim_stoch_gradien} and it is a slight improvement of Lemma~2 in~\citep{mokhtari2018conditional} and Lemma~5 in~\citep{mokhtari20_stoch_condit_gradien_method}. 

\begin{lemma}[Theorem~3 of \citep{chen18_projec_free_onlin_optim_stoch_gradien}]\label{lem:variance_reduction}
Let $\{ \BF{a}_n \}_{n=0}^{N}$ be a sequence of points in $\B{R}^d$ such that $\| \BF{a}_n - \BF{a}_{n-1} \| \leq G_0/(n+s)$ for all $1 \leq n \leq N$ with fixed constants $G_0 \geq 0$ and $s \geq 3$.
Let $\{ \tilde{\BF{a}}_n \}_{n=1}^N$ be a sequence of random variables such that 
$\B{E}[ \tilde{\BF{a}}_n|\C{F}_{n-1} ] = \BF{a}_n$
and 
$\B{E}[ \| \tilde{\BF{a}}_n - \BF{a}_n \|^2 | \C{F}_{n-1} ] \leq \sigma^2$ 
for every $n \geq 0$, where $\C{F}_{n-1}$ is the $\sigma$-field generated by $\{ \tilde{\BF{a}}_i\}_{i=1}^{n}$ and $\C{F}_{0} = \varnothing$.
Let $\{ \BF{d}_n \}_{n=0}^N$ be a sequence of random variables where $\BF{d}_0$ is fixed and subsequent $\BF{d}_{n}$ are obtained by the recurrence 
\begin{equation}
\BF{d}_n = (1-\rho_n) \BF{d}_{n-1} +\rho_n \tilde{\BF{a}}_n
\end{equation}
with $\rho_n = \frac{2}{(n+s)^{2/3}}$. 
Then, we have
\begin{equation}
\B{E}[\| \BF{a}_n - \BF{d}_n \|^2 ] \leq \frac{Q}{(n+s+1)^{2/3}},
\end{equation}
where $Q := \max \{ \|\BF{a}_0 - \BF{d}_0 \|^2 (s+1)^{2/3}, 4\sigma^2 + 3G_0^2/2 \}$.
\end{lemma}

We now analyze the variance of our gradient estimator, which, in the case when we only have access zeroth-order information, uses batched spherical sampling and momentum for gradient estimation.
Calculations similar to the proof of the following Lemma, in the value oracle case, appear in the proof of Theorem~2 in~\cite{chen20_black_box_submod_maxim}.
The main difference is that here we consider a more general smoothing trick and therefore we estimate the gradient along the affine hull of $\C{K}$.

\begin{lemma}\label{L:derivative_estimate_control}
Under the assumptions of Theorem~\ref{T:main_offline}, in Algorithm~\ref{ALG:main_offline}, we have
\begin{align*}
\B{E}\left[ \| \tFgrad(\z_n) - \bar{\g}_n \|^2 \right]
\leq 
\frac{Q}{(n+4)^{2/3}},
\end{align*}
for all $1 \leq n \leq N$ where $\C{L} = \op{aff}(\C{K})$, 
\[
Q = \begin{cases}
0
    &\text{det. grad. oracle},\\
\max \{4^{2/3}G^2,6 L^2 D^2+\frac{4 \sigma_1^2}{B} \}
    &\text{stoch. grad. oracle with variance } \sigma_1^2 > 0,\\
\max \{4^{2/3}G^2,6 L^2 D^2+\frac{4 C k G^2+2 k^2\sigma_0^2/\delta^2}{B} \}
    &\text{value oracle with variance } \sigma_0^2 \geq 0,\\
\end{cases}
\]
$C$ is a constant and $D = \op{diam}(\C{K})$.
\end{lemma}

\begin{remark}
As we will see in the proof of Theorem~\ref{T:main_offline}, except for the case with deterministic gradient oracle, the dominating term in the approximation error is a constant multiple of 
\[
\frac{1}{N}\sum_{n = 1}^{N}\B{E}\left[ \| \tFgrad(\z_n) - \bar{\g}_n \|^2 \right].
\]
Therefore, any improvement in Lemma~\ref{L:derivative_estimate_control} will result in direct improvement of the approximation error.
\end{remark}

\begin{proof}
If we have access to a deterministic gradient oracle, then the claim is trivial.
Let $\C{F}_1 := \varnothing$ and $\C{F}_n$ be the $\sigma$-field generated by $\{\bar{\g}_1, \dots, \bar{\g}_{n-1} \}$ and let
\[
\sigma^2 = \begin{cases}
\frac{\sigma_1^2}{B}
    &\text{stoch. grad. oracle with variance } \sigma_1^2 > 0,\\
\frac{C k G^2 + k^2\sigma_0^2/2\delta^2}{B}
    &\text{value oracle with variance } \sigma_0^2 \geq 0.\\
\end{cases}
\]

Let $\C{L}_0$ denote the linear space $\C{L} - \x$ for some $\x \in \C{L}$.
If we have access to a stochastic gradient oracle, then $\g_n$ is computed by taking the average of $B$ gradient samples of $P_{\C{L}_0}(\hat{G}(\z))$, i.e. the projection of $\hat{G}(\z)$ onto the linear space $\C{L}_0$. 
Since $P_{\C{L}_0}$ is a 1-Lipscitz linear map, we see that 
\[
\B{E}[P_{\C{L}_0}(\hat{G}(\z))] = P_{\C{L}_0}(\nabla \tF(\z)) = \tFgrad(\z)
\]
and 
\begin{align*}
\B{E}\left[ \left\| P_{\C{L}_0}(\hat{G}(\z)) - \tFgrad(\z) \right\|^2\right]
&= \B{E}\left[ \left\| P_{\C{L}_0}(\hat{G}(\z)) - P_{\C{L}_0}(\nabla \tF(\z)) \right\|^2\right] \\
&\leq \B{E}\left[ \left\| \hat{G}(\z) - \nabla \tF(\z) \right\|^2\right]
\leq \sigma_1^2.
\end{align*}
Note that, in cases where we have access to a gradient oracle, we have $\delta = 0$ and $\tF = F$.
Therefore
\begin{align*}
\B{E}\left[ \g_n|\C{F}_{n-1} \right] = \tFgrad(\z_n)
\quad\text{ and }\quad
\B{E}\left[ \|\g_n - \tFgrad(\z_n) \|^2|\C{F}_{n-1} \right]
\leq \frac{\sigma_1^2}{B} = \sigma^2.
\end{align*}
Next we assume that we have access to a value oracle.
By the unbiasedness of $\hat{F}$ and \cref{L:tilde_F_mean_variance}, we have
\begin{align*}
\B{E}\left[ 
    \frac{k}{2\delta}(\hat{F}(\y_{n,i}^+) - \hat{F}(\y_{n,i}^-))\uu_{n,i}|\C{F}_{n-1} 
\right]
=& 
\B{E}\left[
    \B{E}\left[
        \frac{k}{2\delta}(\hat{F}(\y_{n,i}^+) - \hat{F}(\y_{n,i}^-))\uu_{n,i}
        | \C{F}_{n-1}, \uu_{n,i}
    \right] | \C{F}_{n-1}
\right]
\\
=& 
\B{E}\left[ 
    \frac{k}{2\delta}(F(\y_{n,i}^+) - F(\y_{n,i}^-))\uu_{n,i} | \C{F}_{n-1} 
\right]
\\
=& \tFgrad(\z_n),
\end{align*}
and
\begin{align*}
&\hspace{-1cm}\B{E}\left[
    \left\|
        \frac{k}{2\delta}(\hat{F}(\y_{n,i}^+)-\hat{F}(\y_{n,i}^-))\uu_{n,i}-\tFgrad(\z_n)
    \right\|^2 | \C{F}_{n-1}
\right] \\
&= \B{E}\left[ \B{E}\left[\left\| 
    \frac{k}{2\delta}(F(\y_{n,i}^+)-F(\y_{n,i}^-))\uu_{n,i} - \tFgrad(\z_n) 
    \right.\right.\right.\\
    &\qquad\qquad + \frac{k}{2\delta}(\hat{F}(\y_{n,i}^+)-F(\y_{n,i}^+))\uu_{n,i} \\
    &\qquad\qquad \left.\left.\left. - \frac{k}{2\delta} (\hat{F}(\y_{n,i}^-)-F(\y_{n,i}^-))\uu_{n,i} 
\right\|^2 | \C{F}_{n-1},\uu_{n,i} \right] | \C{F}_{n-1} \right] \\
&\leq \B{E}\left[ \B{E}\left[\left\| 
    \frac{k}{2\delta}(F(\y_{n,i}^+)-F(\y_{n,i}^-))\uu_{n,i} - \tFgrad(\z_n)
\right\|^2 | \C{F}_{n-1},\uu_{n,i} \right] | \C{F}_{n-1} \right] \\
&\qquad\qquad + \B{E}\left[ \B{E}\left[\left\| 
    \frac{k}{2\delta}(\hat{F}(\y_{n,i}^+)-F(\y_{n,i}^+))\uu_{n,i}
\right\|^2 | \C{F}_{n-1},\uu_{n,i} \right] | \C{F}_{n-1} \right] \\
&\qquad\qquad + \B{E}\left[ \B{E}\left[\left\| 
    \frac{k}{2\delta} (\hat{F}(\y_{n,i}^-)-F(\y_{n,i}^-))\uu_{n,i}
\right\|^2 | \C{F}_{n-1},\uu_{n,i} \right] | \C{F}_{n-1} \right] \\
&\leq \B{E}\left[\left\| 
    \frac{k}{2\delta}(F(\y_{n,i}^+)-F(\y_{n,i}^-))\uu_{n,i}-\tFgrad(\z_n)
\right\|^2|\C{F}_{n-1}\right] \\
&\qquad\qquad + \frac{k^2}{4\delta^2} \B{E}\left[\B{E}\left[
    | \hat{F}(\y_{n,i}^+)-F(\y_{n,i}^+) |^2
    \cdot 
    \|\uu_{n,i}\|^2
| \C{F}_{n-1},\uu_{n,i} \right] | \C{F}_{n-1} \right] \\
&\qquad+ \frac{k^2}{4\delta^2} \B{E}\left[\B{E}\left[
    | \hat{F}(\y_{n,i}^-)-F(\y_{n,i}^-) |^2
    \cdot \|\uu_{n,i}\|^2
| \C{F}_{n-1}, \uu_{n,i} \right] | \C{F}_{n-1} \right] \\
&\leq C k G^2 + \frac{k^2}{4\delta^2}\sigma_0^2+\frac{k^2}{4\delta^2}\sigma_0^2 \\
&= C k G^2+\frac{k^2}{2\delta^2}\sigma_0^2.
\end{align*}
So we have
\begin{align*}
\B{E}\left[ \g_n|\C{F}_{n-1} \right] 
= \B{E}\left[
    \frac{1}{B}\sum_{i=1}^{B}\frac{k}{2\delta}(\hat{F}(\y_{n,i}^+)-\hat{F}(\y_{n,i}^-))\uu_{n,i} | \C{F}_{n-1}
\right]
= \tFgrad(\z_n),
\end{align*}
and
\begin{align*}
&\B{E}\left[ \left\| \g_n - \tFgrad(\z_n) \right\|^2 | \C{F}_{n-1} \right] \\
&\qquad\qquad= \frac{1}{B^2} \sum_{i=1}^{B} 
\B{E}\left[\left\|
    \frac{k}{2\delta}(\hat{F}(\y_{n,i}^+)-\hat{F}(\y_{n,i}^-))\uu_{n,i} - \tFgrad(\z_n) 
\right\|^2 | \C{F}_{n-1} \right] \\
&\qquad\qquad\leq \frac{C k G^2+\frac{k^2}{2\delta^2}\sigma_0^2}{B} = \sigma^2.
\end{align*}

Using Lemma~\ref{lem:variance_reduction} with $\BF{d}_n = \bar{\g}_n, \tilde{\BF{a}}_n = \g_n, \BF{a}_n = \tFgrad(\z_n)$ for all $n \geq 1$, $\BF{a}_0=\tFgrad(\z_1)$, $G_0 = 2LD$ and $s = 3$, we have
\begin{equation}
\B{E}[\| \tFgrad(\z_n) - \bar{\g}_n\|^2]
\leq \frac{Q'}{(n+4)^{2/3}},
\end{equation}
where $Q' = \max \{ \|\tFgrad(\z_1) \|^2 4^{2/3}, 6 L^2 D^2 + 4\sigma^2 \}$.
Note that by \cref{L:smooth_approx}, we have $\| \tFgrad(x) \| \leq G$, 
thus we have $Q' \leq Q$.
\end{proof}

\section{Proof of Theorem~\ref{T:main_offline} for monotone maps over convex sets containing zero}\label{APX:monotone_dc}

\begin{proof}
By the definition of $\z_n$, we have 
$\z_n = \z_1 + \sum_{i=1}^{n-1} \frac{\vv_i}{N}$.
Therefore $\z_n - \z_1$ is a convex combination of $\vv_n$'s and $0$ which belong to $\C{K}_\delta - \z_1$ and therefore $\z_n - \z_1 \in \C{K}_\delta - \z_1$.
Hence we have $\z_n \in \C{K}_\delta \subseteq \C{K}$ for all $1 \leq n \leq N+1$.

Let $\C{L} := \op{aff}(\C{K})$.
According to Lemma~\ref{L:smooth_approx}, the function $\tF$ is $L$-smooth.
So we have
\begin{equation}
\begin{aligned}
\tF(\z_{n+1}) - \tF(\z_n)
&\geq \bra \tFgrad (\z_n), \z_{n+1} - \z_n \ket
    - \frac{L}{2}\|\z_{n+1} - \z_n\|^2 \\
&= \varepsilon \bra \tFgrad (\z_n), \vv_n \ket
    - \frac{\varepsilon^2 L}{2} \|\vv_n\|^2 \\
&\geq \varepsilon \bra \tFgrad (\z_n), \vv_n \ket
    - \frac{\varepsilon^2 L}{2}D^2 \\
&= \varepsilon \left( \bra \bar{\g}_n, \vv_n \ket + \bra \tFgrad (\z_n) - \bar{\g}_n, \vv_n \ket \right)
    - \frac{\varepsilon^2 L D^2}{2}.
\end{aligned}
\label{EQ:m-dc-1}
\end{equation}
Let $\z^*_\delta := \argmax_{\z \in \C{K}_\delta - \z_1} \tF(z)$.
We have $\z^*_\delta \in \C{K}_\delta - \z_1$, which implies that
$\bra \bar{\g}_n, \vv_n \ket \geq \bra \bar{\g}_n, \z^*_\delta \ket$.
Therefore
\begin{align*}
\bra \bar{\g}_n, \vv_n \ket
\geq \bra \bar{\g}_n, \z^*_\delta \ket 
= \bra \tFgrad(\z_n), \z^*_\delta \ket
    + \bra \bar{\g}_n - \tFgrad(\z_n), \z^*_\delta \ket
\end{align*}
Hence we obtain
\begin{align*}
\bra \bar{\g}_n, \vv_n \ket + \bra \tFgrad (\z_n) - \bar{\g}_n, \vv_n \ket
\geq \bra \tFgrad(\z_n), \z^*_\delta \ket
    - \bra \tFgrad(\z_n) - \bar{\g}_n, \z^*_\delta - \vv_n \ket
\end{align*}
Using the Cauchy-Schwartz inequality, we have
\begin{align*}
\bra \tFgrad(\z_n) - \bar{\g}_n), \z^*_\delta - \vv_n \ket
\leq \|\tFgrad(\z_n)-\bar{\g}_n\| \|\z^*_\delta - \vv_n \|
\leq D \| \tFgrad(\z_n)-\bar{\g}_n\|
\end{align*}
Therefore
\begin{align*}
\bra \bar{\g}_n, \vv_n \ket + \bra \tFgrad (\z_n) - \bar{\g}_n, \vv_n \ket
\geq \bra \tFgrad(\z_n), \z^*_\delta \ket
    - D \| \tFgrad(\z_n)-\bar{\g}_n\|.
\end{align*}
Plugging this into~\ref{EQ:m-dc-1}, we see that
\begin{align}\label{EQ:m-dc-2}
\tF(\z_{n+1}) - \tF(\z_n)
\geq \varepsilon \bra \tFgrad(\z_n), \z^*_\delta \ket
    - \varepsilon D \| \tFgrad(\z_n)-\bar{\g}_n\|
    - \frac{\varepsilon^2 L D^2}{2}.
\end{align}

On the other hand, we have $\z^*_\delta \geq (\z^*_\delta - \z_n) \vee 0$.
Since $F$ is monotone continuous DR-submodular, by Lemma~\ref{L:smooth_approx}, so is $\tF$.
Moreover monotonicity of $\tF$ implies that $\tFgrad$ is non-negative in positive directions.
Therefore we have
\begin{align*}
\bra \tFgrad(\z_n), \z^*_\delta \ket 
&\geq \bra \tFgrad(\z_n), (\z^*_\delta - \z_n) \vee 0 \ket 
    &&(\text{monotonicity}) \\
&\geq \tF(\z_n + ( (\z^*_\delta - \z_n) \vee 0 ))-\tF(\z_n)
    &&(\text{DR-submodularity}) \\
&= \tF(\z^*_\delta \vee \z_n)-\tF(\z_n) \\
&\geq \tF(\z^*_\delta)-\tF(\z_n)
\end{align*}
After plugging this into \eqref{EQ:m-dc-2} and re-arranging terms, we obtain
\begin{align*}
h_{n+1} 
&\leq (1 - \varepsilon) h_n 
    + \varepsilon D \| \tFgrad(\z_n)-\bar{\g}_n\|
    + \frac{\varepsilon^2 L D^2}{2}
\end{align*}
where $h_n := \tF(\z^*_\delta) - \tF(\z_n)$.
After taking the expectation and using Lemma~\ref{L:derivative_estimate_control}, we see that
\begin{align*}
\B{E}(h_{n+1})
&\leq (1 - \varepsilon) \B{E}(h_n)
    + \frac{\varepsilon D Q^{1/2}}{(n+4)^{1/3}}
    + \frac{\varepsilon^2 L D^2}{2}.
\end{align*}
Using the above inequality recursively and $1 - \varepsilon \leq 1$, we have
\begin{align*}
\B{E}[h_{N+1}]
\leq (1 - \varepsilon)^N \B{E}[h_1]
    + \sum_{n=1}^N \frac{\varepsilon D Q^{1/2}}{(n+4)^{1/3}} + \frac{N \varepsilon^2L D^2}{2}.
\end{align*}

Note that we have $\varepsilon = 1/N$.
Using the fact that $(1 - \frac{1}{N})^N \leq e^{-1}$ and
\begin{equation}
\begin{aligned}
\sum_{n=1}^N \frac{D Q^{1/2}}{(n+4)^{1/3}}
\leq D Q^{1/2} \int_0^N \frac{\dif x}{(x+4)^{1/3}}
\leq D Q^{1/2} \left( \frac{3}{2}(N+4)^{2/3} \right) \\
\leq D Q^{1/2} \left( \frac{3}{2} (2 N)^{2/3} \right)
\leq 3 D Q^{1/2} N^{2/3},
\end{aligned}
\label{EQ:Q-term}    
\end{equation}
we see that
\begin{align*}
\B{E}[h_{N+1}]
\leq e^{-1} \B{E}[h_1]
    + \frac{3 D Q^{1/2}}{N^{1/3}}+ \frac{L D^2}{2 N}.
\end{align*}
By re-arranging the terms and using the fact that $\tF$ is non-negative, we conclude
\begin{equation}
\begin{aligned}
(1 - e^{-1}) \tF(\z^*_\delta) - \B{E}[\tF(\z_{N+1})] 
&\leq -e^{-1} \tF(\z_1) 
    + \frac{3 D Q^{1/2}}{N^{1/3}} + \frac{L D^2}{2 N} \\
&\leq \frac{3 D Q^{1/2}}{N^{1/3}} + \frac{L D^2}{2 N}.
\end{aligned}
\label{EQ:m-dc-final-delta}
\end{equation}

According to Lemma~\ref{L:smooth_approx}, we have $\tF(\z_{N+1}) \leq F(\z_{N+1})+\delta G$.
Moreover, using Lemma~\ref{L:shrunk_contrainst_set}, we see that $\z^* \in \B{B}_{\delta'}(\C{K}_\delta)$ where $\delta' = \delta D / r$.
Therefore, there is a point $\y^* \in \C{K}_\delta$ such that $\|\y^* - \z^*\| \leq \delta'$.
\begin{align*}
\tF(\z^*_\delta)
&\geq \tF(\y^* - \z_1)
\geq \tF(\y^*) - G \|\z_1\| \\
&\geq F(\y^*) - (\|\z_1\| + \delta) G
\geq F(\z^*) - (\|\z_1\| + \delta + \frac{\delta D}{r}) G.
\end{align*}
According to Lemma~\ref{L:norm_z_1}, we have $\|\z_1\| \leq \sqrt{d} \|\z_1\|_\infty \leq \delta \sqrt{d}/r$.
\begin{align}\label{EQ:corrected_69_A}
\tF(\z^*_\delta)
\geq F(\z^*) - (1 + \frac{\sqrt{d} + D}{r}) \delta G.
\end{align}
After plugging these into~\ref{EQ:m-dc-final-delta}, we see that
\begin{align*}
(1-e^{-1})F(\z^*) &- \B{E}[F(\z_{N+1})] \\
&\leq \frac{3 D Q^{1/2}}{N^{1/3}} + \frac{L D^2}{2 N} 
    + \delta G(2 + \frac{\sqrt{d} + D}{r}).
\qedhere
\end{align*}
\end{proof}

\section{Proof of Theorem~\ref{T:main_offline} for non-monotone maps over downward-closed convex sets}\label{APX:non-monotone_dc}

\begin{proof}
Similar to~\cref{APX:monotone_dc}, we see that $\z_n \in \C{K}_\delta$ for all $1 \leq n \leq N+1$ and
\begin{align}\label{EQ:nm-dc-1}
\tF(\z_{n+1}) - \tF(\z_n)
\geq \varepsilon \left( \bra \bar{\g}_n, \vv_n \ket + \bra \tFgrad (\z_n) - \bar{\g}_n, \vv_n \ket \right)
    - \frac{\varepsilon^2 L D^2}{2}.
\end{align}
Let $\z^*_\delta := \argmax_{\z \in \C{K}_\delta - \z_1} \tF(z)$.
We have $\z^*_\delta \vee \z_n - \z_n = (\z^*_\delta - \z_n) \vee 0 \leq \z^*_\delta$.
Therefore, since $\C{K}_\delta$ is downward-closed, we have $\z^*_\delta \vee \z_n - \z_n \in \C{K}_\delta - \z_1$.
On the other hand, $\z^*_\delta \vee \z_n - \z_n \leq \BF{1} - \z_n$.
Therefore, we have
$\bra \bar{\g}_n, \vv_n \ket \geq \bra \bar{\g}_n, \z^*_\delta \vee \z_n - \z_n \ket$,
which implies that
\begin{align*}
\bra \bar{\g}_n, \z^*_\delta &\vee \z_n - \z_n \ket 
    + \bra \tFgrad (\z_n) - \bar{\g}_n, \vv_n \ket \\
&= \bra \tFgrad(\z_n), \z^*_\delta \vee \z_n - \z_n \ket 
    + \bra \bar{\g}_n - \tFgrad(\z_n), \z^*_\delta \vee \z_n - \z_n \ket \\
&\qquad+ \bra \tFgrad(\z_n) - \bar{\g}_n, \vv_n \ket \\
&= \bra \tFgrad(\z_n), \z^*_\delta \vee \z_n - \z_n \ket 
    - \bra \tFgrad(\z_n) - \bar{\g}_n, - \vv_n + \z^*_\delta \vee \z_n - \z_n \ket
\end{align*}
Using the Cauchy-Shwarz inequality, we see that
\begin{align*}
\bra \tFgrad(\z_n) - \bar{\g}_n, - \vv_n + \z^*_\delta \vee \z_n - \z_n \ket
&\leq \|\tFgrad(\z_n)-\bar{\g}_n\| \|(\z^*_\delta \vee \z_n - \z_n) - \vv_n\| \\
&\leq D \| \tFgrad(\z_n)-\bar{\g}_n\|.
\end{align*}
where the last inequality follows from the fact that both $\vv_n$ and $\z^*_\delta \vee \z_n - \z_n$ belong to $\C{K}_\delta$.
Therefore
\begin{align*}
\bra \bar{\g}_n, \z^*_\delta \vee \z_n - \z_n \ket 
    &+ \bra \tFgrad (\z_n) - \bar{\g}_n, \vv_n \ket \\
&\geq \bra \tFgrad(\z_n), \z^*_\delta \vee \z_n - \z_n \ket 
    - D \| \tFgrad(\z_n)-\bar{\g}_n\|.
\end{align*}
Plugging this into Equation~\eqref{EQ:nm-dc-1}, we get
\begin{equation}
\begin{aligned}
\tF(\z_{n+1}) &- \tF(\z_n) \\
&\geq \varepsilon \bra \tFgrad(\z_n), \z^*_\delta \vee \z_n - \z_n \ket 
    - \varepsilon D \| \tFgrad(\z_n)-\bar{\g}_n\|
    - \frac{\varepsilon^2 L D^2}{2}.
\end{aligned}
\label{EQ:nm-dc-2}
\end{equation}

Next we show that 
\begin{align}\label{EQ:1-z_n-dc}
1 - \|\z_n\|_\infty \geq (1-\varepsilon)^{n-1}(1 - \frac{\delta}{r}),
\end{align}
for all $1 \leq n \leq N+1$.
We use induction on $n$ to show that for each coordinate $1 \leq i \leq d$, we have $1 - [\z_n]_i \geq (1-\varepsilon)^{n-1}$.
For $n = 1$, the claim follows from Lemma~\ref{L:norm_z_1}.
Assuming that the inequality is true for $n$, using the fact that $\vv_n \leq \BF{1} - \z_n$, we have
\begin{align*}
1 - [\z_{n+1}]_i 
= 1 - [\z_n]_i - \varepsilon [\vv_n]_i
&\geq 1 - [\z_n]_i - \varepsilon (1 - [\z_n]_i) \\
&= (1 - \varepsilon)(1 - [\z_n]_i) 
\geq (1-\varepsilon)^n(1 - \frac{\delta}{r}),
\end{align*}
which completes the proof by induction.

Since $\tF$ is DR-submodular, it is concave along non-negative directions.
Therefore, using Lemma~\ref{L:F_join_non-monotone} and Equation~\eqref{EQ:1-z_n-dc}, we have
\begin{align*}
\bra \tFgrad(\z_n), \z^*_\delta \vee \z_n - \z_n \ket
&\geq \tF(\z^*_\delta \vee \z_n) - \tF(\z_n) \\
&\geq (1 - \|\z_n\|_\infty) \tF(\z^*_\delta) - \tF(\z_n) \\
&\geq (1 - \varepsilon)^{n-1}(1 - \frac{\delta}{r}) \tF(\z^*_\delta) - \tF(\z_n).
\end{align*}
Plugging this into Equation~\eqref{EQ:nm-dc-2}, we get
\begin{align*}
\tF(\z_{n+1}) &- \tF(\z_n) \\
&\geq \varepsilon \left(
        (1 - \varepsilon)^{n-1} (1 - \frac{\delta}{r}) \tF(\z^*_\delta) - \tF(\z_n)
    \right)
    - \varepsilon D \| \tFgrad(\z_n)-\bar{\g}_n\|
    - \frac{\varepsilon^2 L D^2}{2}.
\end{align*}
Taking expectations of both sides and using Lemma~\ref{L:derivative_estimate_control}, we see that
\begin{align*}
\B{E}(\tF(\z_{n+1}))
&\geq (1 - \varepsilon)\B{E}(\tF(\z_n))
    + \varepsilon (1 - \varepsilon)^{n-1} (1 - \frac{\delta}{r}) \tF(\z^*_\delta)
    - \frac{\varepsilon D Q^{1/2}}{(n+4)^{1/3}}
    - \frac{\varepsilon^2 L D^2}{2}.
\end{align*}
Using this inequality recursively and Equation~\eqref{EQ:Q-term}, we get
\begin{align*}
\B{E}(\tF(\z_{N+1}))
&\geq (1 - \varepsilon)^{N}\B{E}(\tF(\z_1))
    + N\varepsilon (1 - \varepsilon)^{N-1} (1 - \frac{\delta}{r}) \tF(\z^*_\delta) \\
&\qquad- \sum_{n = 1}^N \frac{\varepsilon D Q^{1/2}}{(n+4)^{1/3}}
    - \frac{N \varepsilon^2 L D^2}{2} \\
&\geq (1 - \varepsilon)^{N}\B{E}(\tF(\z_1))
    + N \varepsilon (1 - \varepsilon)^{N-1} (1 - \frac{\delta}{r}) \tF(\z^*_\delta) \\
&\qquad- 3 \varepsilon D Q^{1/2} N^{2/3}
    - \frac{N \varepsilon^2 L D^2}{2}.
\end{align*}
Since $\delta < \frac{r}{2}$ and $\varepsilon = 1/N$, we have
\begin{align*}
(1 - \varepsilon)^N
= (1 - \frac{1}{N})(1 - \varepsilon)^{N-1}
\geq \frac{1}{2}(1 - \varepsilon)^{N-1}
\geq \frac{\delta}{r}(1 - \varepsilon)^{N-1}
\geq \frac{\delta}{r}(1 - \varepsilon)^{N-1}.
\end{align*}
Since $\tF$ is non-negative and $G$-Lipschitz, this implies that
\begin{align*}
\B{E}(\tF(\z_{N+1}))
&\geq \frac{\delta}{r}(1 - \varepsilon)^{N-1}\B{E}(\tF(\z_1))
    + (1 - \varepsilon)^{N-1} (1 - \frac{\delta}{r}) \tF(\z^*_\delta) \\
&\qquad- 3 \varepsilon D Q^{1/2} N^{2/3}
    - \frac{N \varepsilon^2 L D^2}{2} \\
&= (1 - \varepsilon)^{N-1} \tF(\z^*_\delta)
    + \frac{\delta}{r}(1 - \varepsilon)^{N-1} (\B{E}(\tF(\z_1)) - \tF(\z^*_\delta)) \\
&\qquad- 3 \varepsilon D Q^{1/2} N^{2/3}
    - \frac{N \varepsilon^2 L D^2}{2} \\
&\geq (1 - \varepsilon)^{N-1} \tF(\z^*_\delta)
    - \frac{\delta}{r}(1 - \varepsilon)^{N-1} GD
    - 3 \varepsilon D Q^{1/2} N^{2/3}
    - \frac{N \varepsilon^2 L D^2}{2}.
\end{align*}
After setting $\varepsilon = 1/N$ and using $(1 - 1/N)^{N-1} \geq e^{-1}$, we see that
\begin{align*}
e^{-1} \tF(\z^*_\delta) - \B{E}(\tF(\z_{N+1}))
&\leq \frac{3 D Q^{1/2}}{N^{1/3}} 
    + \frac{L D^2}{2 N} 
    + \frac{\delta}{r}(1 - \varepsilon)^{N-1} GD \\
&\leq \frac{3 D Q^{1/2}}{N^{1/3}} 
    + \frac{L D^2}{2 N} 
    + \frac{\delta}{r} GD.
\end{align*}
Using the argument presented in~\cref{APX:monotone_dc}, i.e. Lemma~\ref{L:smooth_approx} and Equation~\ref{EQ:corrected_69_A}, we conclude that \begin{equation*}
e^{-1}F(\z^*) - \B{E}[F(\z_{N+1})] 
\leq \frac{3 D Q^{1/2}}{N^{1/3}} + \frac{L D^2}{2 N} 
    + \delta G(2 + \frac{\sqrt{d} + 2 D}{r}).
\qedhere
\end{equation*}
\end{proof}

\section{Proof of Theorem~\ref{T:main_offline} for monotone maps over general convex sets}\label{APX:monotone_g}

\begin{proof}
Using the fact that $\tF$ is $L$-smooth, we have
\begin{equation}
\begin{aligned}
\tF(\z_{n+1}) - \tF(\z_n)
&\geq \bra \tFgrad(\z_n), \z_{n+1} - \z_n \ket
    - \frac{L}{2} \| \z_{n+1} - \z_n \|^2 \\
&= \varepsilon \bra \tFgrad(\z_n), \vv_n - \z_n \ket
    - \frac{\varepsilon^2 L}{2} \| \vv_n - \z_n \|^2 \\
&\geq \varepsilon \bra \tFgrad(\z_n), \vv_n - \z_n \ket
    - \frac{\varepsilon^2 L D^2}{2} \\
&= \varepsilon \left( \bra \bar{\g}_n, \vv_n - \z_n \ket
    + \bra \tFgrad(\z_n) + \bar{\g}_n, \vv_n - \z_n \ket \right)
    - \frac{\varepsilon^2 L D^2}{2}.
\end{aligned}
\label{EQ:m-g-1}
\end{equation}
Let $\z^*_\delta := \argmax_{\z \in \C{K}_\delta} \tF(z)$.
Using the fact that $\bra \bar{\g}_n, \vv_n \ket \geq \bra \bar{\g}_n, \z^*_\delta \ket$ , we have
\begin{align*}
\bra \bar{\g}_n, \vv_n - \z_n \ket
    &+ \bra \tFgrad(\z_n) - \bar{\g}_n, \vv_n - \z_n \ket \\
&\geq \bra \bar{\g}_n, \z^*_\delta - \z_n \ket
    + \bra \tFgrad(\z_n) - \bar{\g}_n, \vv_n - \z_n \ket \\
&= \bra \tFgrad(\z_n), \z^*_\delta - \z_n \ket 
    - \bra \tFgrad(\z_n) - \bar{\g}_n, \z^*_\delta - \vv_n \ket.
\end{align*}
Using the Cauchy-Schwarz inequality, we see that
\begin{align*}
\bra \tFgrad(\z_n) - \bar{\g}_n, \z^*_\delta - \vv_n \ket
\leq \|\tFgrad(\z_n)-\bar{\g}_n\| \|\z^*_\delta - \vv_n\|
\leq D \| \tFgrad(\z_n)-\bar{\g}_n\|.
\end{align*}
Therefore
\begin{align*}
\bra \bar{\g}_n, \vv_n - \z_n \ket
    &+ \bra \tFgrad(\z_n) - \bar{\g}_n, \vv_n - \z_n \ket \\
&\geq \bra \tFgrad(\z_n), \z^*_\delta - \z_n \ket 
    - D \| \tFgrad(\z_n)-\bar{\g}_n\|.
\end{align*}
Plugging this into~\ref{EQ:m-g-1}, we get
\begin{equation}
\begin{aligned}
\tF(\z_{n+1}) - \tF(\z_n)
\geq \varepsilon \bra \tFgrad(\z_n), \z^*_\delta - \z_n \ket 
    - \varepsilon D \| \tFgrad(\z_n)-\bar{\g}_n\|
    - \frac{\varepsilon^2 L D^2}{2}.
\end{aligned}
\label{EQ:m-g-2}
\end{equation}

Using Lemma~\ref{L:F_join_meet_non-monotone} and the fact that $\tF$ is monotone, we see that
\begin{align*}
\bra \tFgrad(\z_n), \z^*_\delta - \z_n \ket
&\geq \tF(\z^*_\delta \vee \z_n) + \tF(\z^*_\delta \wedge \z_n) - 2\tF(\z_n) \\
&\geq \tF(\z^*_\delta) + \tF(\z^*_\delta \wedge \z_n) - 2\tF(\z_n) \\
&\geq \tF(\z^*_\delta) - 2\tF(\z_n).
\end{align*}
After plugging this into~\eqref{EQ:m-g-2}, we get
\begin{align*}
\tF(\z_{n+1}) - \tF(\z_n)
\geq \varepsilon \tF(\z^*_\delta) 
    - 2 \varepsilon \tF(\z_n)
    - \varepsilon D \| \tFgrad(\z_n)-\bar{\g}_n\|
    - \frac{\varepsilon^2 L D^2}{2}.
\end{align*}
After taking the expectation, using Lemma~\ref{L:derivative_estimate_control} and re-arranging the terms, we see that
\begin{align*}
\B{E}[\tF(\z_{n+1})]
\geq (1 - 2 \varepsilon) \B{E}[\tF(\z_n)]
    + \varepsilon \tF(\z^*_\delta) 
    - \frac{\varepsilon D Q^{1/2}}{(n+4)^{1/3}}
    - \frac{\varepsilon^2 L D^2}{2}.
\end{align*}
Using this inequality recursively together with Equation~\eqref{EQ:Q-term} and the fact that $\tF$ is non-negative, we get
\begin{align*}
\B{E}[\tF(\z_{N+1})]
&\geq (1 - 2 \varepsilon)^N \B{E}[\tF(\z_1)]
    + \varepsilon \tF(\z^*_\delta) \sum_{n = 1}^N (1 - 2 \varepsilon)^{N - n} \\
&\qquad - \sum_{n = 1}^N \frac{\varepsilon D Q^{1/2}}{(n+4)^{1/3}}
    - \frac{N \varepsilon^2 L D^2}{2}. \\
&\geq \frac{1}{2}(1 - 2 \varepsilon)^N \B{E}[\tF(\z_1)]
    + \varepsilon \tF(\z^*_\delta) \sum_{n = 1}^N (1 - 2 \varepsilon)^{N - n} \\
&\qquad - 3 \varepsilon D Q^{1/2} N^{2/3}
    - \frac{N \varepsilon^2 L D^2}{2} \\
&= \frac{1}{2}(1 - 2 \varepsilon)^N \B{E}[\tF(\z_1)]
    + \frac{1}{2}(1 - (1 - 2 \varepsilon)^N) \B{E}[\tF(\z^*_\delta)] \\
&\qquad - 3 \varepsilon D Q^{1/2} N^{2/3}
    - \frac{N \varepsilon^2 L D^2}{2} \\
&= \frac{1}{2} \tF(\z^*_\delta)
    - \frac{1}{2} (1 - 2 \varepsilon)^{N} (\tF(\z^*_\delta) - \B{E}[\tF(\z_1)]) \\
&\qquad - 3 \varepsilon D Q^{1/2} N^{2/3}
    - \frac{N \varepsilon^2 L D^2}{2} \\
&\geq \frac{1}{2} \tF(\z^*_\delta)
    - \frac{1}{2} (1 - 2 \varepsilon)^{N} D G 
    - 3 \varepsilon D Q^{1/2} N^{2/3}
    - \frac{N \varepsilon^2 L D^2}{2}.
\end{align*}
Note that
$(1 - \log(N)/N)^{N} \leq e^{-\log(N)} = 1/N$.
Therefore, since $\varepsilon = \log(N)/2N$, we have
\begin{equation}
\begin{aligned}
\B{E}[\tF(\z_{N+1})]
&\geq \frac{1}{2} \tF(\z^*_\delta)
    - \frac{DG}{2N}
    - \frac{3 D Q^{1/2} \log(N)}{2 N^{1/3}}
    - \frac{L D^2 \log(N)^2}{8 N}.
\end{aligned}
\label{EQ:m-g-final-delta}
\end{equation}

According to Lemma~\ref{L:smooth_approx}, we have $\tF(\z_{N+1}) \leq F(\z_{N+1})+\delta G$.
Moreover, using Lemma~\ref{L:shrunk_contrainst_set}, we see that $\z^* \in \B{B}_{\delta'}(\C{K}_\delta)$ where $\delta' = \delta D / r$.
Therefore, there is a point $\y^* \in \C{K}_\delta$ such that $\|\y^* - \z^*\| \leq \delta'$.
\begin{align}\label{EQ:corrected_69_C}
\tF(\z^*_\delta)
\geq \tF(\y^*)
\geq \tF(\y^*)
\geq F(\y^*) - \delta G
\geq F(\z^*) - (\delta + \frac{\delta D}{r}) G.
\end{align}
After plugging these into~\eqref{EQ:m-g-final-delta}, we see that 
\begin{align*}
\frac{1}{2} \tF(\z^*) &- \B{E}[\tF(\z_{N+1})] \\
&\leq \frac{3 D Q^{1/2} \log(N)}{2 N^{1/3}}
    + \frac{4DG + L D^2 \log(N)^2}{8 N}
    + \delta G(2 + \frac{D}{r}).
\end{align*}
which completes the proof.
\end{proof}

\section{Proof of Theorem~\ref{T:main_offline} for non-monotone maps over general convex sets}\label{APX:non-monotone_g}

\begin{proof}
First we show that 
\begin{align}\label{EQ:1-z_n-g}
1 - \|\z_n\|_\infty \geq (1-\varepsilon)^{n-1} (1 - \|\z_1\|_\infty),
\end{align}
for all $1 \leq n \leq N+1$.
We use induction on $n$ to show that for each coordinate $1 \leq i \leq d$, we have $1 - [\z_n]_i \geq (1-\varepsilon)^{n-1} (1 - [\z_1]_i)$.
The claim is obvious for $n = 1$.
Assuming that the inequality is true for $n$, we have
\begin{align*}
1 - [\z_{n+1}]_i 
= 1 - (1 - \varepsilon)[\z_n]_i - \varepsilon [\vv_n]_i
&\geq 1 - (1 - \varepsilon)[\z_n]_i - \varepsilon \\
&= (1 - \varepsilon)(1 - [\z_n]_i) 
\geq (1-\varepsilon)^n (1 - [\z_1]_i),
\end{align*}
which completes the proof by induction.

Let $\z^*_\delta := \argmax_{\z \in \C{K}_\delta} \tF(z)$.
Using the same arguments as in \cref{APX:monotone_g},  we see that
\begin{align*}
\tF(\z_{n+1}) - \tF(\z_n)
\geq \varepsilon \bra \tFgrad(\z_n), \z^*_\delta - \z_n \ket 
    - \varepsilon D \| \tFgrad(\z_n)-\bar{\g}_n\|
    - \frac{\varepsilon^2 L D^2}{2}.
\end{align*}
Using Lemmas~\ref{L:F_join_meet_non-monotone} and~\ref{L:F_join_non-monotone} and Equation~\eqref{EQ:1-z_n-g}, we have
\begin{align*}
\bra \tFgrad(\z_n), \z^*_\delta - \z_n \ket
&\geq \tF(\z^*_\delta \vee \z_n) + \tF(\z^*_\delta \wedge \z_n) - 2\tF(\z_n) \\
&\geq (1 - \|\z_n\|_\infty)\tF(\z^*_\delta) + \tF(\z^*_\delta \wedge \z_n) - 2\tF(\z_n) \\
&\geq (1 - \varepsilon)^{n-1}(1 - \|\z_1\|_\infty)\tF(\z^*_\delta) + \tF(\z^*_\delta \wedge \z_n) - 2\tF(\z_n) \\
&\geq (1 - \varepsilon)^{n-1}(1 - \|\z_1\|_\infty)\tF(\z^*_\delta) - 2\tF(\z_n).
\end{align*}
Therefore
\begin{align*}
\tF(\z_{n+1}) - \tF(\z_n)
\geq \varepsilon (1 - \varepsilon)^{n-1}(1 - \|\z_1\|_\infty)\tF(\z^*_\delta) 
    - 2 \varepsilon \tF(\z_n) \\
    - \varepsilon D \| \tFgrad(\z_n)-\bar{\g}_n\|
    - \frac{\varepsilon^2 L D^2}{2}.
\end{align*}
After taking the expectation, using Lemma~\ref{L:derivative_estimate_control} and re-arranging the terms, we see that
\begin{align}
\B{E}[\tF(\z_{n+1})]
&\geq (1 - 2 \varepsilon) \B{E}[\tF(\z_n)]
    + \varepsilon (1 - \varepsilon)^{n-1}(1 - \|\z_1\|_\infty)\tF(\z^*_\delta) \nonumber\\
    &\qquad - \frac{\varepsilon D Q^{1/2}}{(n+4)^{1/3}}
    - \frac{\varepsilon^2 L D^2}{2}.
\end{align}
Using this inequality recursively together with Equation~\eqref{EQ:Q-term}, we see that
\begin{equation}
\begin{aligned}
\B{E}[\tF(\z_{N+1})]
&\geq \varepsilon (1 - \|\z_1\|_\infty)\tF(\z^*_\delta) \sum_{n = 1}^N (1 - \varepsilon)^{n-1} (1 - 2 \varepsilon)^{N - n} \\
&\qquad + (1 - 2 \varepsilon)^N \B{E}[\tF(\z_1)]
    - \sum_{n = 1}^N \frac{\varepsilon D Q^{1/2}}{(n+4)^{1/3}}
    - \frac{N \varepsilon^2 L D^2}{2}.
\end{aligned}
\label{EQ:nm-g-3}
\end{equation}
Elementary calculations show that $(1 - \frac{c}{N})^{N - 1} \geq e^{-c}$ for $0 \leq c \leq 2$ and $N \geq 4$.
\footnote{For $0 \leq x \leq \frac{1}{2}$, we have $\log(1-x) \geq -x - \frac{x^2}{2} - x^3$. Therefore, for $0 \leq c \leq 2$ and $N \geq 4$, we have $\log(1 - \frac{c}{N}) \geq - \frac{c}{N} - \frac{c^2}{2N^2} - \frac{c^3}{N^3} \geq -\frac{c}{N-1}$.}
Therefore, since $\varepsilon = \log(2)/N$, we have
\begin{align}\label{EQ:nm-g-4}
(1 - 2 \varepsilon)^N 
\geq  e^{- 2\log(2)} (1 - 2 \varepsilon) 
= \frac{1}{4} \left( 1 - \frac{2 \log(2)}{N} \right) 
\geq \frac{1}{4 N}.
\end{align}
On the other hand
\begin{align*}
\varepsilon \sum_{n = 1}^N (1 - 2\varepsilon)^{N - n}(1 - \varepsilon)^{n-1}
&= \varepsilon (1 - 2\varepsilon)^{N - 1} \sum_{n = 1}^N 
    \left( \frac{1 - \varepsilon}{1 - 2 \varepsilon} \right)^{n-1} \\
&\geq \varepsilon  (1 - 2\varepsilon)^{N - 1} \sum_{n = 1}^N (1 + \varepsilon)^{n-1} \\
&= (1 - 2\varepsilon)^{N - 1} ((1 + \varepsilon)^N - 1).
\end{align*}
We have $(1 + \frac{c}{N})^N \geq e^c(1 - \frac{c^2}{2N})$ for $c \geq 0$ and $N \geq 1$.
\footnote{For $x \geq 0$, we have $\log(1 + x) \geq x - \frac{x^2}{2}$ and $-x \geq \log(1 - x)$. 
Therefore 
$N \log(1 + \frac{c}{N}) \geq N (\frac{c}{N} - \frac{c^2}{2N^2}) = c - \frac{c^2}{2N} \geq c + \log(1 - \frac{c^2}{2N})$.}
Therefore
\begin{align*}
\varepsilon \sum_{n = 1}^N (1 - 2\varepsilon)^{N - n}(1 - \varepsilon)^{n-1}
&= (1 - 2\varepsilon)^{N - 1} \left( \left( 1 + \varepsilon \right)^N - 1 \right) \\
&\geq e^{-2 \log(2)} \left( \left( 1 + \frac{\log(2)}{N} \right)^N - 1 \right) \\
&\geq e^{-2 \log(2)} \left( e^{\log{2}} \left( 1 - \frac{\log(2)^2}{2N} \right) - 1 \right) \\
&= \frac{1}{4} \left( 1 - \frac{\log(2)^2}{N} \right) \\
&\geq \frac{1}{4} - \frac{1}{4N}.
\end{align*}
Plugging this and~\ref{EQ:nm-g-4} into~\ref{EQ:nm-g-3} and using the fact that $\tF(z_1)$ is non-negative, we get
\begin{align*}
\B{E}[\tF(\z_{N+1})]
&\geq \left( \frac{1}{4} - \frac{1}{4N} \right) (1 - \|\z_1\|_\infty)\tF(\z^*_\delta) 
    + \frac{1}{4N} \B{E}[\tF(\z_{1})]
    - \frac{3 D Q^{1/2}}{N^{1/3}}
    - \frac{L D^2}{2 N} \\
&\geq \frac{1}{4} (1 - \|\z_1\|_\infty)\tF(\z^*_\delta) 
    + \frac{1}{4N} \left( \B{E}[\tF(\z_{1})] - \tF(\z^*_\delta)  \right)
    - \frac{3 D Q^{1/2}}{N^{1/3}}
    - \frac{L D^2}{2 N} \\
&\geq \frac{1}{4} (1 - \|\z_1\|_\infty)\tF(\z^*_\delta) 
    - \frac{3 D Q^{1/2}}{N^{1/3}}
    - \frac{D G + 2 L D^2}{4 N}.
\end{align*}
Using the same argument as in~\cref{APX:monotone_g}, we obtain
\begin{equation*}
\frac{1}{4} (1 - \|\z_1\|_\infty)F(\z^*) - \B{E}[F(\z_{N+1})] 
\leq \frac{3 D Q^{1/2}}{N^{1/3}} + \frac{D G + 2 L D^2}{4 N} 
    + \delta G(2 + \frac{D}{r}).
\qedhere
\end{equation*}
\end{proof}

\section{Proof of Theorem~\ref{T:offline-complexity}}\label{APX:offline-complexity}

\begin{proof}
Let $T = O(BN)$ denote the number of evaluations\footnote{We have $T = BN$ when we have access to a gradient oracle and $T = 2BN$ otherwise.} and let $\C{E}_\alpha := \alpha F(\z^*) - \B{E}[F(\z_{N+1})]$ denote the $\alpha$-approximation error.
We prove Cases 1-4 separately.
Note that $F$ being non-monotone or $\BF{0} \in \C{K}$ correspond to cases~\ref{T:offline-monotone_dc},~\ref{T:offline-non-monotone_dc} and~\ref{T:offline-non-monotone_g} of Theorem~\ref{T:main_offline} where $\log(N)$ does not appear in the approximation error bound, which is why $\tilde{O}$ can be replaced with $O$.

\noindent {\bf Case 1 (deterministic gradient oracle): } In this case, we have $Q = \delta = 0$.
According to Theorem~\ref{T:main_offline}, in cases~\ref{T:offline-monotone_dc},~\ref{T:offline-non-monotone_dc} and~\ref{T:offline-non-monotone_g}, the approximation error is bounded by $\frac{D G + 2 L D^2}{4 N} = O(N^{-1})$, and thus we choose $T = N = \Theta(1/\epsilon)$ to get $\C{E}_\alpha = O(\epsilon)$.
Similarly, in case~\ref{T:offline-monotone_g}, we have
 \[
 \C{E}_\alpha
\leq 
\frac{4DG + L D^2 \log(N)^2}{8 N}
= O(N^{-1}\log(N)^2).
\]
We choose $T = N = \Theta(\log^2(\epsilon)/\epsilon)$ to bound $\alpha$-approximation error by $O(\epsilon)$.

\noindent {\bf Case 2 (stochastic gradient oracle): } In this case,  we have $Q = \Theta (1)$ and $\delta = 0$.
According to Theorem~\ref{T:main_offline}, in cases~\ref{T:offline-monotone_dc},~\ref{T:offline-non-monotone_dc} and~\ref{T:offline-non-monotone_g}, the approximation error is bounded by 
\[
\frac{3 D Q^{1/2}}{N^{1/3}} + \frac{D G + 2 L D^2}{4 N}
= O(N^{-1/3} + N^{-1})= O(N^{-1/3}),
\]
so we choose $N = \Theta(1/\epsilon^3)$, $B = 1$ and $T = \Theta(1/\epsilon^3)$ to get $\C{E}_\alpha = O(\epsilon)$.
Similarly, in case~\ref{T:offline-monotone_g}, we have
\begin{align*}
\C{E}_\alpha
\leq \frac{3 D Q^{1/2} \log(N)}{2 N^{1/3}}
+ \frac{4DG + L D^2 \log(N)^2}{8 N} = O(N^{-1/3}\log(N) + N^{-1}\log(N)^2)
\end{align*}
Since $\C{E}_\alpha
\leq  O(N^{-1/3}\log(N)^2)$, we choose $N = \Theta(\log^6(\epsilon)/\epsilon^3)$, $B = 1$ and  $T = \Theta(\log^6(\epsilon)/\epsilon^3)$ to bound $\alpha$-approximation error by $O(\epsilon)$.

\noindent {\bf Case 3 (deterministic value oracle): } In this case,  we have $Q = \Theta(1)$ and $\delta \neq 0$.
According to Theorem~\ref{T:main_offline}, in cases~\ref{T:offline-monotone_dc},~\ref{T:offline-non-monotone_dc} and~\ref{T:offline-non-monotone_g}, the approximation error is bounded by 
\[
\frac{3 D Q^{1/2}}{N^{1/3}} + \frac{D G + 2 L D^2}{4 N} + O(\delta)
= O(N^{-1/3} + \delta),
\]
so we choose $\delta = \Theta(\epsilon)$, $N = \Theta(1/\epsilon^3)$, $B = 1$ and $T = \Theta(1/\epsilon^3)$ to get $\C{E}_\alpha = O(\epsilon)$.
Similarly, in case~\ref{T:offline-monotone_g}, we have
\begin{align*}
\C{E}_\alpha
\leq \frac{3 D Q^{1/2} \log(N)}{2 N^{1/3}}
+ \frac{4DG + L D^2 \log(N)^2}{8 N} + O(\delta) = O(N^{-1/3}\log(N)^2 + \delta).
\end{align*}
We choose $\delta = \Theta(\epsilon)$, $N = \Theta(\log^6(\epsilon)/\epsilon^3)$, $B = 1$ and  $T = \Theta(\log^6(\epsilon)/\epsilon^3)$ to bound $\alpha$-approximation error by $O(\epsilon)$.

\noindent {\bf Case 4 (stochastic value oracle): } In this case,  we have $Q = O(1) + O(\frac{1}{\delta^2 B})$ and $\delta \neq 0$. According to Theorem~\ref{T:main_offline}, in cases~\ref{T:offline-monotone_dc},~\ref{T:offline-non-monotone_dc} and~\ref{T:offline-non-monotone_g}, the approximation error is bounded by 
\begin{align*}
\frac{3 D Q^{1/2}}{N^{1/3}} + \frac{D G + 2 L D^2}{4 N} + O(\delta)
&= O(Q^{1/2}N^{-1/3} + N^{-1} + \delta) \\
&= O(N^{-1/3} + \delta^{-1} B^{-1/2} N^{-1/3} + \delta),
\end{align*}
so we choose $\delta = \Theta(\epsilon)$, $N = \Theta(1/\epsilon^3)$, $B = \Theta(1/\epsilon^2)$ and $T = \Theta(1/\epsilon^5)$ to get $\C{E}_\alpha = O(\epsilon)$.
Similarly, in case~\ref{T:offline-monotone_g}, we have
\begin{align*}
\C{E}_\alpha
&\leq \frac{3 D Q^{1/2} \log(N)}{2 N^{1/3}} + \frac{4DG + L D^2 \log(N)^2}{8 N} + O(\delta) \\
&= O(Q^{1/2}N^{-1/3}\log(N) + N^{-1}\log(N)^2 + \delta) \\
&= O(N^{-1/3}\log(N)^2 + \delta^{-1} B^{-1/2} N^{-1/3}\log(N)^2 + \delta).
\end{align*}
We choose $\delta = \Theta(\epsilon)$, $N = \Theta(\log^6(\epsilon)/\epsilon^3)$, $B = \Theta(1/\epsilon^2)$, and $T = \Theta(\log^6(\epsilon)/\epsilon^5)$ to bound $\alpha$-approximation error by $O(\epsilon)$.
\end{proof}

\section{Proof of Theorem~\ref{T:offline-to-online}} \label{APX:offline-to-online}

\begin{proof}
Since the parameters of Algorithm~\ref{ALG:main_offline} are chosen according to Theorem~\ref{T:offline-complexity}, we see that the $\alpha$-approximation error is bounded by $\tilde{O}(T_0^{-\beta})$ where $\beta = 1/3$ in case~2 (stochastic gradient oracle) and $\beta = 1/5$ in case~4 (stochastic value oracle).

Recall that $F$ is $G$-Lipschitz and the feasible region $\C{K}$ has diameter $D$.
Thus, during the first $T_0$ time-steps, the $\alpha$-regret can be bounded by
\[
\sup_{z, z' \in \C{K}} \alpha F(\z) - F(\z') 
\leq 
\sup_{z, z' \in \C{K}} F(\z) - F(\z') 
\leq DG.
\]
Therefore the total $\alpha$-regret is bounded by
\begin{align*}
T_0 DG + (T - T_0) \tilde{O}(T_0^{-\beta})
\leq T_0 DG + T \tilde{O}(T_0^{-\beta}).
\end{align*}
Since we have $T_0 = \Theta(T^{\frac{1}{\beta + 1}})$, we see that
\[
T_0 DG + T \tilde{O}(T_0^{-\beta})
= \tilde{O}(T^{\frac{1}{\beta + 1}})
= \begin{cases}
\tilde{O}(T^{\frac{3}{4}})     &\text{Case 2,} \\
\tilde{O}(T^{\frac{5}{6}})     &\text{Case 4.}
\end{cases}
\]
If $F$ is non-monotone or $\BF{0} \in \C{K}$, the exact same argument applies with $\tilde{O}$ replaced by $O$.
\end{proof}

\end{document}